\documentclass[10pt,twocolumn,letterpaper]{article}

\usepackage{cvpr}
\usepackage{times}
\usepackage{epsfig}
\usepackage{graphicx}
\usepackage{amsmath}
\usepackage{amssymb}

\usepackage{amsthm}
\usepackage{caption}
\usepackage{subcaption}
\usepackage{bm}
\usepackage{stmaryrd}
\usepackage{pifont}
\newcommand{\cmark}{\ding{51}}

\newtheorem{theorem}{Theorem}

\usepackage[breaklinks=true,bookmarks=false]{hyperref}

\cvprfinalcopy

\def\cvprPaperID{1790}

\begin{document}

\title{Continual Learning with Extended Kronecker-factored Approximate Curvature}

\author{Janghyeon Lee\textsuperscript{\rm 1} \qquad Hyeong Gwon Hong\textsuperscript{\rm 2} \qquad Donggyu Joo\textsuperscript{\rm 1} \qquad Junmo Kim\textsuperscript{\rm 1,2}\\
\textsuperscript{\rm 1}School of Electrical Engineering, KAIST\\
\textsuperscript{\rm 2}Graduate School of AI, KAIST\\
{\tt\small \{wkdgus9305, honggudrnjs, jdg105, junmo.kim\}@kaist.ac.kr}
}

\maketitle

\begin{abstract}
We propose a quadratic penalty method for continual learning of neural networks that contain batch normalization (BN) layers.
The Hessian of a loss function represents the curvature of the quadratic penalty function, and a Kronecker-factored approximate curvature (K-FAC) is used widely to practically compute the Hessian of a neural network.
However, the approximation is not valid if there is dependence between examples, typically caused by BN layers in deep network architectures.
We extend the K-FAC method so that the inter-example relations are taken into account and the Hessian of deep neural networks can be properly approximated under practical assumptions.
We also propose a method of weight merging and reparameterization to properly handle statistical parameters of BN, which plays a critical role for continual learning with BN, and a method that selects hyperparameters without source task data.
Our method shows better performance than baselines in the permuted MNIST task with BN layers and in sequential learning from the ImageNet classification task to fine-grained classification tasks with ResNet-50, without any explicit or implicit use of source task data for hyperparameter selection.
\end{abstract}

\section{Introduction}

An artificial neural network catastrophically forgets about what it has learned from previous tasks in a sequential learning scenario as it is optimized to solve its current target problem only~\cite{french1999catastrophic, mccloskey1989catastrophic}.
Although many continual learning methods that aim to mitigate catastrophic forgetting have been introduced in recent years, their use in real-world applications remains somewhat limited owing to their unclear hyperparameter settings or their use of out-of-date network architectures.
Specifically, many of the current continual learning methods do not explain how to determine hyperparameters, or experiments are usually performed with hyperparameters carefully tuned through a validation over whole tasks including past tasks, which is generally not possible in practice~\cite{pfulb2018a}.
Further, even though most state-of-the-art deep network architectures~\cite{he2016deep, huang2017densely} cannot be easily trained without batch normalization (BN) layers~\cite{ioffe2015batch}, the effect of BN is usually not considered and out-of-date network architectures which do not require BN are still mostly used to evaluate continual learning methods.

BN is quite tricky to manipulate due to training and evaluation discrepancy, and its statistical parameters, means and variances of activations.
Those statistical parameters are naturally model's parameters learned from training data, but they are not free parameters controlled by gradient descent; they are just determined by the distribution of data.
Especially, in a sequential learning scenario, they are vulnerable to get beyond control as we are dealing with multiple different tasks whose distributions are not the same in general.
Thus, the effect of BN must be discussed and considered carefully.

In this paper, we propose a quadratic penalty method with a Hessian approximation, for continual learning of neural networks that contain BN layers.
Before delving into the detailed method, let us briefly introduce and review how Hessian can be used for continual learning.
For sequential learning, one of the oracle methods which gives an upper bound performance is the multitask learning method described as the following optimization problem,
\begin{equation} \label{eq:mtl}
\underset{\theta}{\text{min}} \quad \lambda_s \mathbb{E}_{(x, t) \in \mathcal{D}_s}[\mathcal{L}(x, t; \theta)] + \lambda_t \mathbb{E}_{(x, t) \in \mathcal{D}_t}[\mathcal{L}(x, t; \theta)] ,
\end{equation}
where $\mathcal{D}_s$ and $\mathcal{D}_t$ are source and target training dataset, respectively; $\mathcal{L}$ is a loss function; $\lambda_s$ and $\lambda_t$ are importance hyperparameters that control how important each task is; and $\theta$ is a vectorized model's parameter.
As $\mathcal{D}_s$ is not available in a sequential learning scenario, we consider approximating the source task loss function with a function of only $\theta$, excluding data $x$ and label $t$.
In particular, because target task optimization starts with a source task solution $\theta^*$ as the initial position, it is natural to take the second-order Taylor series approximation of the source task loss function at $\theta = \theta^*$; one advantage of this is that the first-order term disappears because the gradient at a local minimum is zero.
Thus, if we denote the gradient and Hessian of the source task loss function at $\theta = \theta^*$ averaged over the source dataset as
\begin{align}
g & = \mathbb{E}_{(x, t) \in \mathcal{D}_s}\left[\left. \frac{\partial\mathcal{L}(x, t; \theta)}{\partial \theta} \right|_{\theta = \theta^*}  \right],\\
H & = \mathbb{E}_{(x, t) \in \mathcal{D}_s}\left[\left. \frac{\partial^2 \mathcal{L}(x, t; \theta)}{\partial \theta \partial \theta} \right|_{\theta = \theta^*}  \right] ,
\end{align}
then
\begin{equation}\begin{split} \label{eq:sourceloss}
\mathbb{E}&_{(x, t) \in \mathcal{D}_s}[\mathcal{L}(x, t; \theta)]\\
& \simeq \frac{1}{2} (\theta - \theta^*)^\top H (\theta - \theta^*) + g^\top (\theta - \theta^*) + C\\
& = \frac{1}{2} (\theta - \theta^*)^\top H (\theta - \theta^*) + C
\end{split}\end{equation}
where $C$ is some constant.
The gradient $g$ is simply a zero vector because $\theta^*$ is a local minimum of the source task loss function.
Therefore, combining Equations~\ref{eq:mtl} and \ref{eq:sourceloss}, we get the following continual learning optimization problem which approximates the multitask learning objective.
\begin{equation} \label{eq:cl}
\underset{\theta}{\text{min}} \quad \lambda_t \mathbb{E}_{(x, t) \in \mathcal{D}_t}[\mathcal{L}(x, t; \theta)] + \frac{1}{2} \lambda_s (\theta - \theta^*)^\top H (\theta - \theta^*) 
\end{equation}
The former term is a usual training loss for the target task and the latter is a quadratic penalty term centered at the initial point $\theta^*$ with the curvature of $H$.

This approximation of the multitask learning objective was used for Hessian pseudopattern backpropagation~\cite{french2002using}, and it is actually the same as elastic weight consolidation (EWC)~\cite{kirkpatrick2017overcoming} and online structured Laplace approximation (OSLA)~\cite{ritter2018online} except that they are grounded in Bayesian approaches.
EWC uses diagonal entries of the Hessian only, and later, in OSLA, a Kronecker-factored block diagonal Hessian approximation is adopted to enable a significant improvement in the performance from EWC.
However, in EWC and OSLA, the effect of BN was not considered so only shallow networks were used for evaluation, and hyperparameters were found by validations over whole tasks.

We describe our contributions in two parts.
In Section~\ref{sec:H}, for theoretical background, we demonstrate that the current Hessian approximation method is not valid when a network has BN layers as dependence between examples caused by BN is not considered, and we extend the Hessian approximation method by taking into account the inter-example relations, so that the Hessian of such network can be validly approximated.
In Section~\ref{sec:method}, for actual continual learning, we propose the detailed methods including how to apply the proposed quadratic penalty loss with BN layers, how to handle statistical parameters, and how to select hyperparameters without source task data.

\section{Related work}

\textbf{Curvature approximation.}
Various methods for Kronecker factorization of curvature have been studied over the past few years.
\cite{martens2015optimizing} proposes an efficient method for approximating natural gradient descent in neural networks, Kronecker-factored approximate curvature (K-FAC), which is derived by approximating blocks of the Fisher as being the Kronecker product of two much smaller matrices.
Similarly, \cite{botev2017practical} presents an efficient block diagonal approximation to the Gauss--Newton matrix for second order optimization of neural networks.
\cite{ba2016distributed} develops a version of K-FAC that distributes the computation of gradients and additional quantities required by K-FAC across multiple machines.
In \cite{grosse2016kronecker}, a tractable approximation to the Fisher matrix for convolutional networks, Kronecker factors for convolution (KFC), is derived based on a structured probabilistic model for the distribution over backpropagated derivatives.
For recurrent neural networks, \cite{martens2018kroneckerfactored} extends the K-FAC method by modelling the statistical structure between the gradient contributions at different time-steps.
\cite{george2018fast} introduces Eigenvalue-corrected Kronecker factorization, an approximate factorization of the Fisher that is computationally manageable, accurate, and amendable to cheap partial updates.

\section{Curvature approximation} \label{sec:H}

\subsection{Notation}

The $l$-th learnable linear layer in any feedforward neural network can be described as
\begin{equation}
h_l = W_l \bar{a}_{l-1},
\qquad
l = 1, 2, \cdots, L
\label{eq:linearlayer}
\end{equation}
where $W_l$ is a $C_l \times (C_{l-1}+1)$ weight matrix and $\bar{a}_{l-1}$ is a $(C_{l-1}+1) \times N$ activation matrix with a homogeneous dimension appended; in other words, bias terms are absorbed into $W_l$ by appending an all-ones vector to $a_{l-1}$.
The whole network takes $x = a_0$ as its input of size $C_0 \times N$ and produces the corresponding output $y$ of size $C_L \times N$, where $N$ is the size of a mini-batch.
We denote the loss of the $n$-th example in a random mini-batch sample $(x, t)$ as $\mathcal{L}_n (x, t)$, or simply $\mathcal{L}_n$, when it is clear from the context, so that we can distinguish between an example index $n$ and a mini-batch sample $(x, t)$.
Thus, the optimization objective can be written as $\mathbb{E}_{(x, t)} [\mathbb{E}_n[\mathcal{L}_n]]$.
In this paper, we denote the $(i, j)$-th entry of a matrix $A$ as $(A)_{i, j}$, and the $(i, j)$-th block of a block matrix $A$ as $\{A\}_{i, j}$.

\subsection{K-FAC}

The Hessian of the linear layer described by Equation~\ref{eq:linearlayer}, considering all inter-example relations, is represented by the following theorem.
\begin{theorem} \textbf{(The Hessian of a linear layer)}
\begin{multline} \label{eq:hessian}
\frac{\partial^2 \mathcal{L}_n}{\partial (W_l)_{a,b} \partial (W_{l'})_{c,d}} \\
= \sum_{m,m'} (\bar{a}_{l-1})_{b,m} (\bar{a}_{l'-1})_{d,m'} \frac{\partial^2 \mathcal{L}_n}{\partial (h_l)_{a,m} \partial (h_{l'})_{c,m'}}
\end{multline}
\end{theorem}
\begin{proof}
See Appendix~\ref{app:hessian}.
\end{proof}
If a network does not have any BN layers, then $\mathcal{L}_n$ depends on only the $n$-th example, and therefore,
\begin{multline} \label{eq:nobn}
\mathbb{E}_{(x, t)} \left[ \mathbb{E}_n \left[\frac{\partial^2 \mathcal{L}_n}{\partial (W_l)_{a,b} \partial (W_{l'})_{c,d}}\right]\right]\\
= \mathbb{E}_{(x, t)} \left[ \mathbb{E}_n \left[ (\bar{a}_{l-1})_{b,n} (\bar{a}_{l'-1})_{d,n} \frac{\partial^2 \mathcal{L}_n}{\partial (h_l)_{a,n} \partial (h_{l'})_{c,n}} \right]\right] .
\end{multline}
As calculating and saving all entries of the Hessian is infeasible in practice owing to its size even for a simple network (\eg, the Hessian of a 1024$\times$1024 fully connected layer takes 4 TB in memory with single-precision floating-point numbers), the K-FAC method factors the Hessian into two relatively small matrices~\cite{botev2017practical, grosse2016kronecker, martens2018kroneckerfactored, martens2015optimizing} so that Equation~\ref{eq:nobn} is approximated by
\begin{multline} \label{eq:nobnapprox}
\mathbb{E}_{x} \left[ \mathbb{E}_n \left[ (\bar{a}_{l-1})_{b,n} (\bar{a}_{l'-1})_{d,n} \right]\right]\\
\cdot \mathbb{E}_{(x, t)} \left[ \mathbb{E}_n \left[ \frac{\partial^2 \mathcal{L}_n}{\partial (h_l)_{a,n} \partial (h_{l'})_{c,n}} \right]\right] .
\end{multline}
This approximation assumes that there would be a small correlation between the two very different variables, one from activations and the other from gradients, or their joint distribution is well approximated by a multivariate Gaussian so that their higher-order joint cumulants would be small~\cite{martens2015optimizing}.
If we define block matrices $H$, $\bar{A}$, and $\mathcal{H}$ as
\begin{align}
(\{H\}_{l,l'})_{i,j} & = \mathbb{E}_{(x, t)} \left[ \mathbb{E}_n \left[ \frac{\partial^2 \mathcal{L}_n}{\partial (\text{vec}(W_l))_i \partial (\text{vec}(W_{l'}))_j}\right]\right] ,\\
(\{\bar{A}\}_{l,l'})_{b,d} & = \mathbb{E}_{x} \left[ \mathbb{E}_n \left[ (\bar{a}_{l-1})_{b,n} (\bar{a}_{l'-1})_{d,n} \right]\right] ,\\
(\{\mathcal{H}\}_{l,l'})_{a,c} & = \mathbb{E}_{(x, t)} \left[ \mathbb{E}_n \left[ \frac{\partial^2 \mathcal{L}_n}{\partial (h_l)_{a,n} \partial (h_{l'})_{c,n}} \right]\right] ,
\end{align}
then Equation~\ref{eq:nobnapprox}, K-FAC, can be written in a simple block matrix form \cite{botev2017practical, martens2015optimizing} as
\begin{equation} \label{eq:nobnsimple}
H \approx \bar{A} * \mathcal{H}
\end{equation}
where $*$ is the Khatri--Rao product.

\subsection{Extended K-FAC}

Unfortunately, for a network with BN layers, each example in a mini-batch affects each other, and therefore, $\frac{\partial^2 \mathcal{L}_n}{\partial (h_l)_{a,m} \partial (h_{l'})_{c,m'}}$ is not zero in general, even if $m \neq n$ and $m' \neq n$, and all summands in Equation~\ref{eq:hessian} remain to be computed.
If each summand is approximated by K-FAC, it costs the amount of memory and computation about $N^2$ times more; this could be prohibitive for a large mini-batch size.
Hence, we need to take a closer look at and make use of the properties of Equation~\ref{eq:hessian} and BN.

For the $n$-th example, note that the other examples are statistically indistinguishable if mini-batches are sampled uniformly from the dataset. For example, for any $n' \neq n$, let $x'$ be the mini-batch obtained by interchanging the $n$-th and $n'$-th example indices of $x$. Then, for any function $f$,
\begin{align}
\mathbb{E}_{x} [ f( (\bar{a}_{l-1})_{b,n} )] &= \mathbb{E}_{x'} [ f( (\bar{a}_{l-1})_{b,n'}) ] \\
&= \mathbb{E}_{x} [ f((\bar{a}_{l-1})_{b,n'} )] ,
\end{align}
where the second equality comes from the fact that $\mathbb{E}_{x} =  \mathbb{E}_{x'}$ when all possible mini-batches of the dataset are sampled equally likely.

Thus, the $N^2$ summands in Equation~\ref{eq:hessian} can be divided into five statistically distinct groups: (i) $m=m'=n$, (ii) $m=m'\neq n$, (iii) $m=n \neq m'$, (iv) $m \neq n = m'$, and (v) $n \neq m \neq m' \neq n$.
If we apply K-FAC in each of the five groups, then we get the following theorem.

\begin{theorem} \textbf{(Extended K-FAC)}
Let $\pi$ and $\pi'$ be permutations of $\{1, 2, \cdots, N\}$, and assume that mini-batches are sampled equally likely from all possible mini-batches of the dataset.
If
\begin{multline}
\mathbb{E}_{(x, t)} \Bigg[ \mathbb{E}_n \Bigg[ (\bar{a}_{l-1})_{b,\pi(n)} (\bar{a}_{l'-1})_{d,\pi'(n)}\\
\cdot \frac{\partial^2 \mathcal{L}_n}{\partial (h_l)_{a,\pi(n)} \partial (h_{l'})_{c,\pi'(n)}} \Bigg]\Bigg]
\end{multline}
and
\begin{multline}
\mathbb{E}_{x} \left[ \mathbb{E}_n \left[ (\bar{a}_{l-1})_{b,\pi(n)} (\bar{a}_{l'-1})_{d,\pi'(n)} \right]\right]\\
\cdot \mathbb{E}_{(x, t)} \left[ \mathbb{E}_n \left[ \frac{\partial^2 \mathcal{L}_n}{\partial (h_l)_{a,\pi(n)} \partial (h_{l'})_{c,\pi'(n)}} \right]\right]
\end{multline}
are the same for any $l$, $l'$, $a$, $b$, $c$, $d$, $\pi$, and $\pi'$, then
\begin{equation} \label{eq:bnsimple}
H = \bar{A} * \mathcal{H}' + \frac{1}{\textnormal{max}(N-1, 1)} \left( N \bar{A}' -  \bar{A} \right) * (\mathcal{H}'' - \mathcal{H}') ,
\end{equation}
where 
\begin{align}
(\{\bar{A}'\}_{l,l'})_{b,d} & = \mathbb{E}_{x} [ \mathbb{E}_n [(\bar{a}_{l-1})_{b,n}] \mathbb{E}_n [(\bar{a}_{l'-1})_{d,n}] ] ,\\
(\{\mathcal{H}'\}_{l,l'})_{a,c} & = \mathbb{E}_{(x, t)} \left[ \mathbb{E}_n \left[ \sum_{m} \frac{\partial^2 \mathcal{L}_n}{\partial (h_l)_{a,m} \partial (h_{l'})_{c,m}} \right]\right] ,\\
(\{\mathcal{H}''\}_{l,l'})_{a,c} & = \mathbb{E}_{(x, t)} \left[ \mathbb{E}_n \left[ \sum_{m,m'} \frac{\partial^2 \mathcal{L}_n}{\partial (h_l)_{a,m} \partial (h_{l'})_{c,m'}} \right]\right] .
\end{align}
\end{theorem}
\begin{proof}
See Appendix~\ref{app:bnsimple}.
\end{proof}

We will call Equation~\ref{eq:bnsimple} the extended K-FAC, or simply XK-FAC.
Note that $\mathcal{H} = \mathcal{H}' = \mathcal{H}''$ for a BN-free case, and hence, XK-FAC is a good generalization of K-FAC (Equation~\ref{eq:nobnsimple}).
Further, it can be applied with a large mini-batch size in practice because it requires only twice as much memory and computation as the original K-FAC.

If the Hessian is approximated by the Fisher information matrix~\cite{pascanu2013revisiting} to avoid negative eigenvalues in practice, then $\hat{\mathcal{H}}'$ and $\hat{\mathcal{H}}''$ are used instead of $\mathcal{H}'$ and $\mathcal{H}''$, where
\begin{align}
(\{\hat{\mathcal{H}}'\}_{l,l'})_{a,c} & = \mathbb{E}_{(x, y)} \left[ \mathbb{E}_n \left[ \sum_{m} \frac{\partial \mathcal{L}_n}{\partial (h_l)_{a,m}}  \frac{\partial \mathcal{L}_n}{\partial (h_{l'})_{c,m}} \right]\right] ,\\
(\{\hat{\mathcal{H}}''\}_{l,l'})_{a,c} & = \mathbb{E}_{(x, y)} \Bigg[ \mathbb{E}_n \Bigg[ \left( \sum_{m} \frac{\partial \mathcal{L}_n}{\partial (h_l)_{a,m}}  \right) \nonumber \\
&\hspace{60px}\cdot\left( \sum_{m} \frac{\partial \mathcal{L}_n}{\partial (h_{l'})_{c,m}} \right) \Bigg]\Bigg] .
\end{align}
They are easily computed from the gradients, where the expectations are taken over the model's distribution.
In this case, it is guaranteed that XK-FAC is positive semi-definite by the following theorem.
\begin{theorem} \textbf{(Positive semi-definiteness of XK-FAC)}
\begin{equation}
\bar{A} * \hat{\mathcal{H}}' + \frac{1}{\textnormal{max}(N-1, 1)} \left( N \bar{A}' -  \bar{A} \right) * (\hat{\mathcal{H}}'' - \hat{\mathcal{H}}') \succeq 0
\end{equation}
\end{theorem}
\begin{proof}
See Appendix~\ref{app:psd}.
\end{proof}

\section{Method for continual learning} \label{sec:method}

Though XK-FAC is one of the key elements, XK-FAC alone cannot achieve continual learning with BN.
Here we propose additional solutions to using BN properly with continual learning settings and discuss issues with hyperparameter selection.

\subsection{Merged weight and batch renormalization} \label{sec:mer}

A weight matrix can be expressed by the product of multiple weight matrices in non-unique ways and consecutive weight matrices can be merged into one matrix, so a single network can be parameterized in many different ways.
If there are many equivalent parameterization methods, then it would be natural to take the simplest form for efficiency.
In particular, it will be computationally efficient and reduce the size of the Hessian if we can treat multiple layers as a single layer.

Let us consider a BN layer with its learnable affine parameters $\gamma$ and $\beta$.
First, for any input $z$, by observing that
\begin{equation}
\frac{z_{jn} - \bar{\mu}_i}{\sqrt{\bar{\sigma}_i^2 + \epsilon}} \gamma_i + \beta_i = \frac{\gamma_i}{\sqrt{\bar{\sigma}_i^2 + \epsilon}} z_{jn} + \left( \beta_i - \frac{\gamma_i \bar{\mu}_i}{\sqrt{\bar{\sigma}_i^2 + \epsilon}} \right) ,
\end{equation} 
we propose merging the normalization part and the affine transformation part of a BN layer into one equivalent affine transformation layer with its data-dependent learnable parameters $\tilde{\gamma}$ and $\tilde{\beta}$ defined as
\begin{equation}
\tilde{\gamma}_i = \frac{\gamma_i}{\sqrt{\bar{\sigma}_i^2 + \epsilon}},
\qquad
\tilde{\beta}_i = \beta_i - \frac{\gamma_i \bar{\mu}_i}{\sqrt{\bar{\sigma}_i^2 + \epsilon} } ,
\end{equation}
where $\bar{\mu}_i$ and $\bar{\sigma}_i^2$ are mini-batch mean and variance, respectively.

If a preceding linear layer exists with its weight $w$ (excluding bias), then this linear layer and the BN layer are merged into one equivalent fully connected layer as well.
Thus, we define a merged fully connected layer whose data-dependent merged weight $\tilde{w}$ and merged bias $\tilde{b}$ are
\begin{equation} \label{eq:mergebn}
\tilde{w}_{ij} = \frac{\gamma_i w_{ij}}{\sqrt{\bar{\sigma}_i^2 + \epsilon}},
\qquad
\tilde{b}_i = \beta_i - \frac{\gamma_i \bar{\mu}_i}{\sqrt{\bar{\sigma}_i^2 + \epsilon} } .
\end{equation}
Then, $\tilde{w}$ and $\tilde{b}$ are concatenated to form a merged weight matrix $\tilde{W}$.
By this reparameterization, we can consider the penalty loss $\frac{1}{2} (\tilde{\theta} - \tilde{\theta}^*)^\top H (\tilde{\theta} - \tilde{\theta}^*)$ on the merged parameter $\tilde{\theta} = \text{vec}(\tilde{W})$, where the Hessian $H$ is also taken with respect to $\tilde{\theta}$.
As $\tilde{W}$ can represent any real matrix and is a continuous function of $w$, $\gamma$, and $\beta$, $\tilde{W}$ can move in any direction in the parameter space by moving $w$, $\gamma$, and $\beta$ appropriately.
Thus, we do not lose any representation power of the original network from the reparameterization.

In addition to efficiency, there are other advantages of reparameterization.
If a single weight matrix is factorized to multiple matrices and each factor is penalized by Equation~\ref{eq:sourceloss}, then the factors are penalized to be kept close to their own original values.
However, their original values individually are less important because of the non-uniqueness of factorization, and it is enough to keep the value of their multiplication only.
Without reparameterization, each factor is unnecessarily over-penalized.

Further, the proposed reparameterization method plays one of the most important roles in continual learning with BN's statistical parameters $\bar{\mu}$ and $\bar{\sigma}$ which are determined by data and preceding weight matrices.
If only free parameters are penalized and statistical parameters are not, then the penalty loss is totally not capable of preserving the performance of the original network, even at the global minimum of the penalty loss, due to drifts in statistics.
So one may attempt to also directly penalize $\bar{\mu}$ and $\bar{\sigma}$, but this approach will fail because all preceding weights are unnecessarily penalized to keep the original statistics, which extremely conflicts with the penalty loss for the preceding weights.
Statistical parameters and weight values cannot simultaneously keep their original values unless the statistics of source and target task data are the same.
In contrast, if $\bar{\mu}$ and $\bar{\sigma}$ are merged with free parameters such as $w$, $\gamma$, and $\beta$, then free parameters can compensate for changes in $\bar{\mu}$ and $\bar{\sigma}$ caused by domain changes, while keeping the merged parameter similar without affecting preceding weights.

The use of mini-batch statistics causes stochastic fluctuations in $\tilde{W}$ because of the dependence of $\bar{\mu}$ and $\bar{\sigma}$ on mini-batch sampling, which make the penalty loss keep oscillating even at its local minimum.
Hence, the penalty loss will converge more stably if such mini-batch dependencies can be eliminated; there is already a good solution for this scenario: batch renormalization (BRN)~\cite{ioffe2017batch}.
BRN uses population statistics instead of mini-batch statistics also in the training mode, while retaining the benefits of the BN by selective gradient propagation. Therefore, by adopting BRN in our framework, the merged parameters become
\begin{equation} \label{eq:mergebrn}
\tilde{w}_{ij} = \frac{\gamma_i w_{ij}}{\sqrt{\bar{\sigma}_i^2 + \epsilon}} r_i,
\quad
\tilde{b}_i = \beta_i + \gamma_i d_i - \frac{\gamma_i \bar{\mu}_i}{\sqrt{\bar{\sigma}_i^2 + \epsilon} } r_i ,
\end{equation}
where
\begin{equation}
r_i = \sqrt{\frac{\bar{\sigma}_i^2 + \epsilon}{\sigma_i^2 + \epsilon}} ,
\qquad
d_i =  \frac{\bar{\mu}_i - \mu_i}{\sqrt{\sigma_i^2 + \epsilon}}
\end{equation}
are treated as constants so that the gradient is not propagated through $r_i$ and $d_i$, and where $\mu_i$ and $\sigma_i^2$ are population statistics.

\subsection{Forward/backward passes of the penalty loss}

The gradient of the penalty loss can be efficiently calculated by matrix multiplications when the Hessian is approximated by a Kronecker-factored form~\cite{martens2018kroneckerfactored, ritter2018online}.
If $\mathcal{L}_s = \frac{1}{2} (\tilde{\theta} - \tilde{\theta}^*)^\top H (\tilde{\theta} - \tilde{\theta}^*)$ is the penalty loss, then for XK-FAC,
\begin{multline} \label{eq:bp}
\frac{\partial \mathcal{L}_s}{\partial \tilde{W}_l} = \sum_{l'} \Bigg( \{\mathcal{H}'\}_{l, l'} (\tilde{W}_{l'} - \tilde{W}_{l'}^*) \{\bar{A}\}_{l, l'}^\top\\
+ \{\mathcal{H}'' - \mathcal{H}'\}_{l, l'} (\tilde{W}_{l'} - \tilde{W}_{l'}^*) \left\{\frac{N\bar{A}' - \bar{A}}{\text{max}(N-1,1)}\right\}_{l, l'}^\top \Bigg) ,
\end{multline}
since $(A \otimes B) \text{vec}(C) = \text{vec}(B C A^\top)$.
Then, the forward pass can be performed by the following identity.
\begin{equation}
\mathcal{L}_s = \frac{1}{2} \sum_{l} \text{vec}(\tilde{W}_{l} - \tilde{W}_{l}^*) ^\top \text{vec}\left(\frac{\partial \mathcal{L}_s}{\partial \tilde{W}_l}\right)
\end{equation}
Finally, $\frac{\partial \mathcal{L}_s}{\partial \tilde{w}}$ and $\frac{\partial \mathcal{L}_s}{\partial \tilde{b}}$ are obtained by detaching the homogeneous dimension from $\frac{\partial \mathcal{L}_s}{\partial \tilde{W}}$, and they are propagated to $\frac{\partial \mathcal{L}_s}{\partial w}$, $\frac{\partial \mathcal{L}_s}{\partial \gamma}$, $\frac{\partial \mathcal{L}_s}{\partial \beta}$, and $\frac{\partial \mathcal{L}_s}{\partial a}$ by the usual chain rule.

\subsection{Damping}

The second-order approximation in Equation~\ref{eq:sourceloss} is trustworthy only in a local neighborhood around the initial point.
The quadratic penalty loss can prevent the parameters from moving too far from the initial point when positive semi-definite approximations of Hessian such as the Fisher information matrix~\cite{pascanu2013revisiting} or the generalized Gauss--Newton matrix~\cite{schraudolph2002fast} are used; however, the problem can be still ill-posed.
That is, if some eigenvalue of a Hessian is very close to zero, the parameter can move along the direction of the corresponding eigenvector with almost no restriction and eventually escape the trust region.

Here, we rethink the weight decay method, which is commonly used to regularize a network.
The weight decay loss drags parameters down to the origin so that any single parameter cannot easily dominate the others and hopefully a network generalizes well.
However, in our continual learning framework, the weight decay loss interferes with the penalty loss that attempts to maintain parameters near the initial point.
Hence, we change the center of the weight decay loss from the origin to the initial point $\theta^*$.
This modified weight decay is equivalent to adding a damping matrix $\lambda I$ to $H$, and it achieves the effect of setting a lower limit $\lambda$ on eigenvalues.

\subsection{Preprocessing and postprocessing}

When the first target sample passes through a network to begin continual learning, data-dependent parameters are changed immediately from the source task statistics to target task statistics.
If the source and target tasks are not similar, then $\tilde{\theta}$ can be very different from $\tilde{\theta}^*$ at the beginning of training, which leads to a poor Hessian approximation of the source task loss function.
To compensate for this drastic change in parameters, we reinitialize $\gamma$ and $\beta$ one time at the beginning of the training for each target task using the following preprocessing method.
Let $\mu^*$ and $\sigma^*$ be the source task population statistics stored in the source network, and $\mu$ and $\sigma$ be the target task ones, which can be obtained by passing target samples through the network in the evaluation mode.
Then, $\gamma$ and $\beta$ are reinitialized to
\begin{equation}
\beta_i \leftarrow \beta_i^* + \frac{\gamma_i^*}{\sqrt{\sigma_i^{*2} + \epsilon}} (\mu_i - \mu_i^*),
\quad
\gamma_i \leftarrow \sqrt{\frac{\sigma_i^2 + \epsilon}{\sigma_i^{*2} + \epsilon}} \gamma_i^* ,
\end{equation}
so that $\tilde{\theta}$ is initialized to $\tilde{\theta}^*$ at the start of a new task, because
\begin{equation}
\frac{z_{jn} - \mu_i^*}{\sqrt{\sigma_i^{*2} + \epsilon}} \gamma_i^* + \beta_i^* = \frac{z_{jn} - \mu_i}{\sqrt{\sigma_i^{2} + \epsilon}} \gamma_i + \beta_i
\end{equation}
for any $z$.

After finishing continual learning with the current source and target tasks, the union of those two tasks becomes the new source task for future continual learning.
Thus, the source task Hessian $H$ must be updated to contain current target task information by postprocessing.
If we let $H_t$ be the Hessian of the current target task loss function, then, from Equation~\ref{eq:cl}, it is obvious that
\begin{equation}
H \leftarrow \lambda_s H + \lambda_t H_t .
\end{equation}

\subsection{Hyperparameter selection} \label{sec:hyp}

A continual learning method has importance hyperparameters or balancing parameters that control the importance of each task or trade-off between tasks.
The choice of such hyperparameters is one of the most important in continual learning because it directly affects the final performance.
However, as \cite{pfulb2018a} pointed out, many approaches do not explain why hyperparameters were selected to have such specific values and how to select them in a real-world case, or they presciently use the best hyperparameters selected after all tasks have been processed, which can implicitly violate causality.

In our method, finding importance hyperparameters $\lambda_s$ and $\lambda_t$ using the source and target data is the same as optimizing the magnitude of the penalty loss, and thus, they should not be picked by validation over whole tasks.
From Equations~\ref{eq:mtl} and \ref{eq:cl}, it is obvious that $\lambda_s$ and $\lambda_t$ directly represent the magnitude of the source and target loss, respectively.
Therefore, if each task is equally important, then the importance hyperparameters are set to
\begin{equation} \label{eq:imp_hyp}
\lambda_s = \frac{T}{T+1}, \qquad\quad \lambda_t = \frac{1}{T+1}
\end{equation}
without any validation, where $T$ is the number of tasks the model has learned so far.
One can also easily set different importance for each task if needed.

If the best learning rate and damping hyperparameter for each task are found by validation over whole tasks, the source task data is then being used to set a trust region of the source task loss function because a small learning rate implicitly restricts the range that parameters can move and the damping method explicitly restricts it.
Thus, for the learning rate and damping hyperparameter $\lambda$, we also avoid using the source task data to find them.
However, there is not a natural choice for such hyperparameters, unlike the importance hyperparameters.

One advantage of the penalty method is that we can always access an approximation of the source task loss function, $\mathcal{L}_s = \frac{1}{2} (\theta - \theta^*)^\top H (\theta - \theta^*)$, and we propose a simple heuristic approach to make use of it.
As we still do not know about source task accuracy, our basic idea is to adaptively scale the source task penalty loss $\lambda_s \mathcal{L}_s$ and target task loss $\lambda_t \mathcal{L}_t$ during training so that they converge to similar values, and to expect that the source and target accuracy drops also would be similar.
However, because $\mathcal{L}_s$ does not contain the constant term in Equation~\ref{eq:sourceloss}, a constant $C_t$ has to be subtracted from $\mathcal{L}_t$ for proper comparison, where $C_t$ is a local minimum value for the target loss function that can be obtained from the fine-tuning with the target task.

To keep $\lambda_s \mathcal{L}_s$ and $\lambda_t (\mathcal{L}_t - C_t)$ similar during training, we introduce adaptive scaling hyperparameters $\alpha_s$ and $\alpha_t$, and minimize $\alpha_s^{-1} \lambda_s \mathcal{L}_s + \alpha_t^{-1} \lambda_t \mathcal{L}_t$ instead of $\lambda_s \mathcal{L}_s + \lambda_t \mathcal{L}_t$.
To be brief, $\alpha_s$ and $\alpha_t$ are initialized to 1, and one of them adaptively increases by 1 until $\lambda_s \mathcal{L}_s$ and $\lambda_t (\mathcal{L}_t - C_t)$ have similar values, and they are re-initialized to 1 to repeat the procedure.
Specifically, if the average value of $\lambda_s \mathcal{L}_s$ within some interval is greater than that of $\lambda_t (\mathcal{L}_t - C_t)$, then $\alpha_t$ increases by 1 to relatively increase the scale of $\lambda_s \mathcal{L}_s$.
$\alpha_t$ continues to increase in this manner at each interval until the average value of $\lambda_s \mathcal{L}_s$ within the interval is eventually smaller than that of $\lambda_t (\mathcal{L}_t - C_t)$, and then, $\alpha_t$ is re-initialized to 1.
In the opposite case, $\alpha_s$ increases by 1 instead of $\alpha_t$, and the procedure is similarly applied in the opposite way.

After training with various settings of learning rates and damping parameters, the model with the best target validation accuracy is determined as the final model.
As the source loss and target loss are minimized while having similar values during training, we hope that the model with the best target validation accuracy will also perform well with the source task.
In this way, we can find hyperparameters by leveraging the penalty loss value, without any access to source data.

\begin{figure*}[t]
\centering
\begin{subfigure}{0.49\linewidth}
\includegraphics[width=\linewidth]{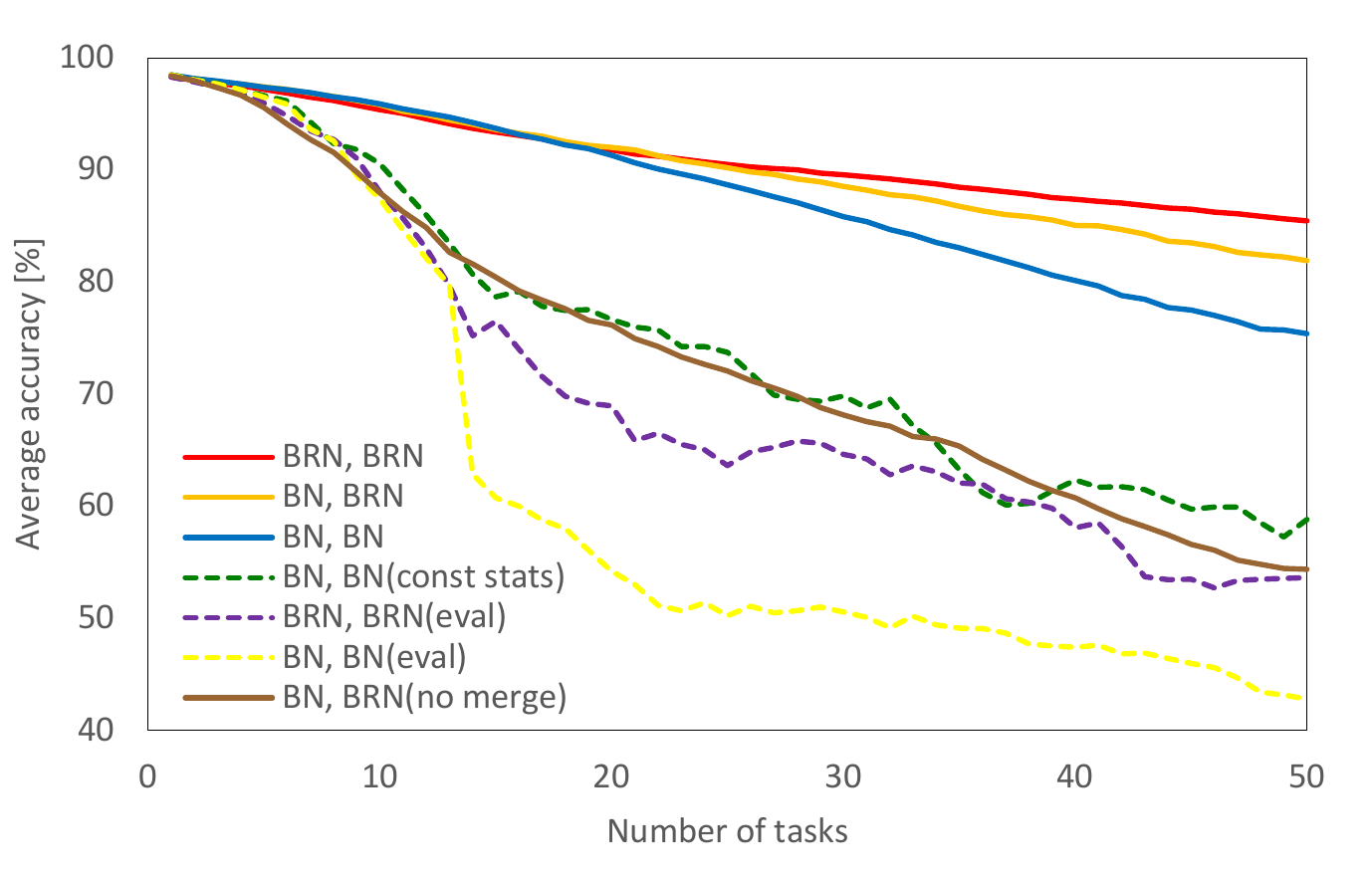}
\caption{}
\label{fig:mnist_a}
\end{subfigure}
\begin{subfigure}{0.49\linewidth}
\includegraphics[width=\linewidth]{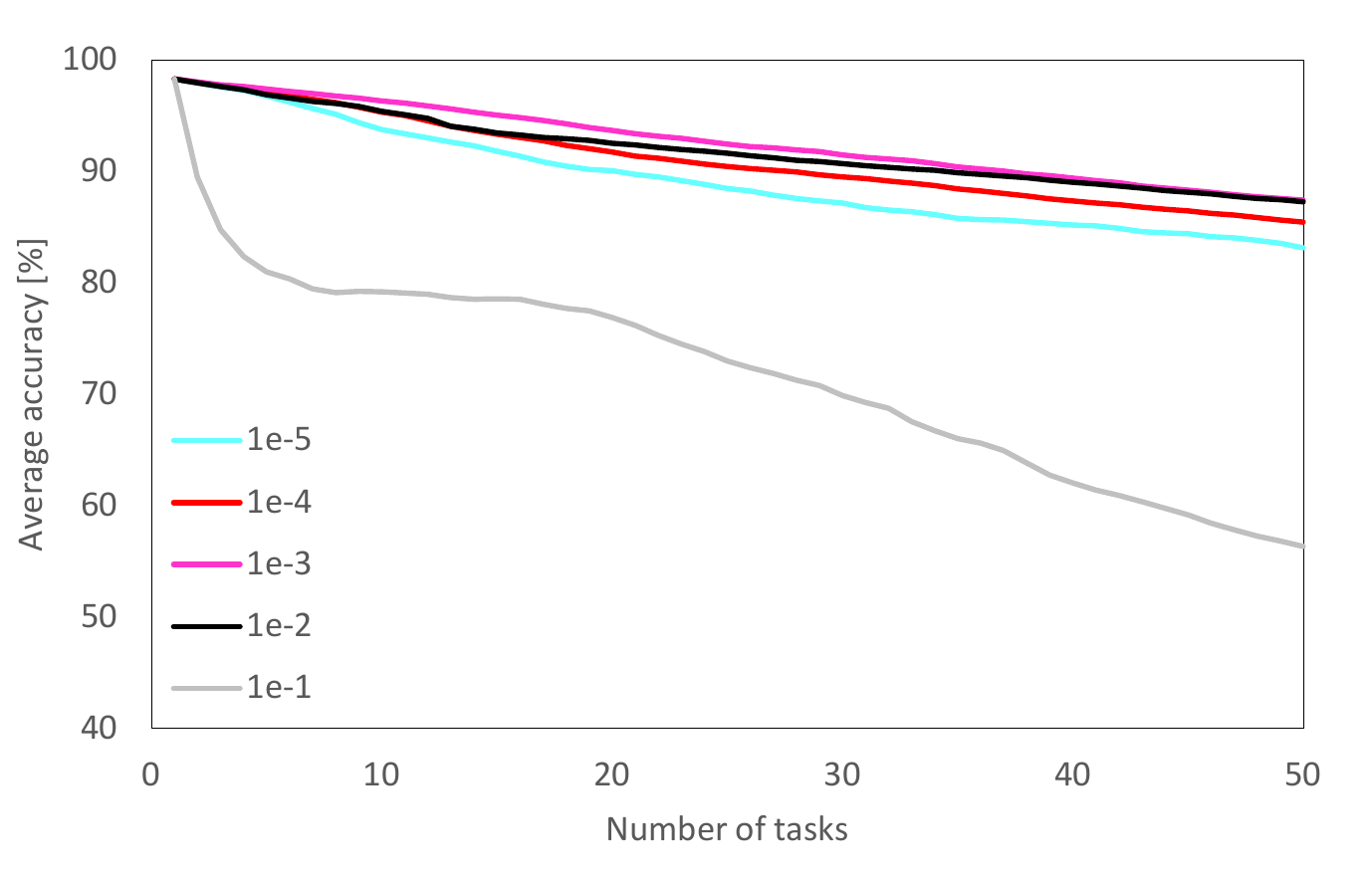}
\caption{}
\label{fig:mnist_b}
\end{subfigure}
\vspace{-5px}
\caption{
Average validation accuracy on permuted MNIST sequential learning.
(a) The legend indicates whether BN layers or BRN layers are used for the network architecture and how their weights are interpreted when constructing the merged weights.
For instance, for the orange line (BN, BRN), BN layers are used for the network, but their weights are interpreted as those of BRN layers when constructing the merged weights.
XK-FAC is used for solid lines, and K-FAC is used for dashed lines.
(b) The legend indicates the value of the damping hyperparameter $\lambda$.
Each experiment was repeated four times.
}
\end{figure*}

\section{Experiment}

\subsection{Permuted MNIST}

\textbf{Task and network architecture.}
As a first experiment, the permuted MNIST~\cite{lecun1998gradient} task is performed as in \cite{goodfellow2013empirical, kirkpatrick2017overcoming, lee2017overcoming, ritter2018online}.
Each task is generated by a fixed random permutation of the original MNIST images.
To investigate the effect of BN, we use a multilayer perceptron (784-128-128-10) with a BN or BRN layer inserted before every ReLU nonlinearity.
BN is actually not necessary here, but even for such shallow network there is a significant performance gap between different BN settings in continual learning, and hence we can demonstrate the effectiveness of the proposed method with this simple experiment.

\textbf{Hyperparameters.}
We emphasize that any source task data is not used explicitly or implicitly for hyperparameter selection.
For all tasks, the same learning rate schedule, optimizer, and damping parameter are used.
The learning rate starts at 0.1, and it is divided by 10 for every 5 epochs, a SGD optimizer with momentum 0.9 and mini-batch size 128 optimizes networks until 15 epochs, and the damping parameter $\lambda$ is set to 0.0001, for all tasks.
Further, we use early stopping based on the validation set of the current target task only, and not the whole tasks.
The importance hyperparameters are set according to Equation~\ref{eq:imp_hyp}.

\textbf{BN and BRN.}
If the merged weights of the BN or BRN layers are constructed in a standard BN-way (Equation~\ref{eq:mergebn}) or BRN-way (Equation~\ref{eq:mergebrn}), XK-FAC is used due to the mini-batch dependence of $\bar{\mu}$ and $\bar{\sigma}$.
In short sequential learning, up to about 20 tasks, both BN (the blue line in Figure~\ref{fig:mnist_a}) and BRN (the red line in Figure~\ref{fig:mnist_a}) networks work similarly.
However, as the sequence becomes extremely long, BRN networks prevent more forgetting than BN networks.
If a model has learned a large number of tasks, it is much more important to maintain source task accuracy than to solve a target task.
Thus, BRN would be preferred for a very long sequence because the penalty loss is more stabilized.

Another interesting finding is that even if BN layers are used in a network architecture, it is better to interpret their weights as those of BRN layers when constructing the merged weights (the orange line in Figure~\ref{fig:mnist_a}).
For mini-batch sizes large enough to assume that population statistics and mini-batch statistics are very similar, a BN and the corresponding BRN layer are almost identical, and thus, the Hessian of BRN layers can still provide a reliable penalty loss to BN layers in a more stabilized way.

Further, we tried other ways of constructing merged weights that circumvent the inter-example relations.
First, mini-batch statistics $\bar{\mu}$ and $\bar{\sigma}$ can be treated as if they were constants (`const stats' in Figure~\ref{fig:mnist_a}).
Or, the population statistics $\mu$ and $\sigma$, which are used in the evaluation mode and considered as constants, can be used to obtain the merged weight instead of the mini-batch statistics (`eval' in Figure~\ref{fig:mnist_a}).
However, the performance drops rapidly if inter-example relations are not considered.

\begin{table*}[t]
  \centering
  \begin{tabular}{|l||c|c|c||c|c|c|c|}
    \hline
    Method & active BN & valid curvature & adaptive $\alpha_s$, $\alpha_t$ & ImageNet & Birds & Cars & Flowers \\
    \hline\hline
    Fixed BN &  & \cmark  &  & 29.21 & 67.05 & 74.24 & 74.20 \\
    K-FAC & \cmark & &  & 50.37 & 70.49 & \textbf{85.98} & 88.45 \\
    XK-FAC & \cmark  & \cmark &  & \textbf{61.15} & \textbf{78.97} & 85.91  & \textbf{91.84} \\
    \hline
    Fixed BN &  & \cmark & \cmark & 49.57 & 76.24 & 84.94 & 81.30 \\
    K-FAC & \cmark  &  & \cmark & 65.31 & 80.09 & 87.35  & 92.14 \\
    XK-FAC & \cmark & \cmark & \cmark & \textbf{66.57} & \textbf{80.30} & \textbf{88.07}  & \textbf{92.56} \\
    \hline
  \end{tabular}
\vspace{-5px}
\caption{
Top-1 validation accuracy on ImageNet to fine-grained classification tasks.
The results were averaged over 6 possible orderings of the three fine-grained datasets.
\vspace{-10px}
}
\label{tab:imagenet}
\end{table*}

\textbf{Effect of weight merging.}
In Section~\ref{sec:mer}, we discussed why it is preferable to use merged weights, especially when there exist statistical parameters such as $\bar{\mu}$ and $\bar{\sigma}$ of BN.
If we individually penalize each parameter, $w$, $\bar{\mu}$, $\bar{\sigma}$, $\gamma$, and $\beta$ instead of the merged weight $\tilde{W}$, each penalty loss will conflict with each other to keep statistical parameters, resulting in significant performance degradation (the brown line in Figure~\ref{fig:mnist_a}).

\textbf{Damping.}
The effect of the damping hyperparameter $\lambda$ is shown in Figure~\ref{fig:mnist_b}.
For moderately small $\lambda$, the damping method can slightly improve the performance.
If it is too large, the Hessian is severely corrupted and cannot give a good approximation of the source task loss function.
We simply set $\lambda$ to a typical value of 0.0001 in the experiment; however, further investigation is needed for determining the value of $\lambda$ or better damping methods in future work.

\subsection{ImageNet to fine-grained classification tasks}

\textbf{Task and network architecture.}
For an experiment with deep networks, we perform sequential learning from the ImageNet classification task to multiple fine-grained classification tasks with ResNet-50~\cite{he2016deep}, as in \cite{mallya2018piggyback, mallya2018packnet}.
Starting with a network pre-trained on ImageNet~\cite{russakovsky2015imagenet}, three fine-grained datasets, CUBS~\cite{WahCUB_200_2011}, Stanford Cars~\cite{krause20133d}, and Flowers~\cite{nilsback2008automated}, are learned sequentially with 6 different orders.
All datasets are preprocessed and augmented in the same way as \cite{mallya2018packnet}.

\textbf{Hyperparameters.}
We also do not use any source task data for hyperparameter selection here.
For all experiments, SGD with momentum 0.9 and mini-batch size 32 optimizes networks until 100 epochs, and $\lambda_s$ and $\lambda_t$ are set according to Equation~\ref{eq:imp_hyp}.
The learning rate gradually decreases to 1e-4, and $\alpha_s$ and $\alpha_t$ are updated every 10 iterations.
We tried initial learning rates of \{1e-1, 5e-2, 2e-2, 1e-2\} and damping hyperparameters of \{1e-4, 5e-5, 2e-5, 1e-5\} for each method, and the final model is determined according to the method described in Section~\ref{sec:hyp}.
For the proposed method, BN layers in a ResNet are treated as BRN layers when constructing the merged weights.
As in \cite{ioffe2017batch}, $r_\text{max}$ and $d_\text{max}$ in BRN are initialized to 1 and 0, and they are gradually relaxed to 3 and 5, respectively.

\textbf{Convolutional network.}
To apply the proposed method to convolutional networks, we combine our XK-FAC and KFC~\cite{grosse2016kronecker}.
As KFC only handles block diagonal parts of the Hessian, the curvature of the penalty loss is approximated by a block diagonal matrix, by considering only the case $l = l'$, as in \cite{ritter2018online}.
Details can be found in Appendix~\ref{app:kfc}.

\textbf{Fixed BN baseline.}
One of the natural baseline methods is a `fixed BN' method.
For a fixed BN method, the parameters of BN, $\gamma$, $\beta$, $\mu$, and $\sigma$, are fixed to the parameters of the ImageNet pre-trained model, and they are not considered as learnable parameters during sequential learning.
It is not very restrictive to fix BN parameters in terms of representation power of a network, because $w$ is still a free parameter and a network does not lose any degree of freedom.
In this case, K-FAC is a valid curvature approximation since there is no inter-example relations.
However, as seen in Table~\ref{tab:imagenet}, fixed BN baselines do not perform well in both source and target tasks.
The curvature of the original source task loss function with BN layers will differ from the Hessian of the corresponding fixed-BN model.
On the other hand, for target tasks, the network cannot take advantages of the BN layers, such as the ease of optimization of such a deep network or better generalization.
We found that networks diverge for learning rates of 1e-1 or 5e-2 due to the absence of BN layers, so we tried learning rates of \{2e-2, 1e-2, 5e-3, 2e-3\} only for the fixed BN method.

\textbf{K-FAC baseline.}
Another baseline is a `K-FAC' method, which does not fix BN but uses an invalid curvature approximation, K-FAC, rather than a valid one, XK-FAC.
For the K-FAC baselines in Table~\ref{tab:imagenet}, the experimental settings and the initial network are exactly the same with that of the XK-FAC methods except the Hessian approximation, so all of the performance degradation comes from ignoring inter-example relations when approximating the curvature.
It is remarkable that simply using a better Hessian approximation can improve performance.
Quantitatively, for the initial network, the mean absolute values of all block diagonal entries in K-FAC and XK-FAC are respectively 5.65e-7 and 4.96e-7, and that of the difference between K-FAC and XK-FAC is 1.19e-7, which is not negligible.

\textbf{Effect of $\alpha_s$ and $\alpha_t$.}
In Section~\ref{sec:hyp}, we proposed a simple approach to finding hyperparameters by leveraging the penalty loss value without any source data, and the special hyperparameters $\alpha_s$ and $\alpha_t$ were introduced to control the balance between a source task penalty loss and a target task loss.
To observe the effect of $\alpha_s$ and $\alpha_t$, we also experimented with $\alpha_s$ and $\alpha_t$ fixed to 1.
As seen in Table~\ref{tab:imagenet}, even though the proposed method of controlling $\alpha_s$ and $\alpha_t$ is very simple, it can significantly improve the performance of all methods, including the baselines.
See Appendix~\ref{app:res} for more results.
It would be worth exploring more sophisticated methods as a future work.

\section{Conclusion}

In this paper, we extended K-FAC to consider inter-example relations caused by BN layers.
Further, detailed methods of applying XK-FAC for continual learning, such as weight merging, use of BRN layers, and hyperparameter selection, were proposed, and we demonstrated that each proposed method works effectively in continual learning.

\section*{Acknowledgment}

This research was supported by the Engineering Research Center Program through the National Research Foundation of Korea (NRF) funded by the Korean Government MSIT (NRF-2018R1A5A1059921)
and was supported by
Institute of Information \& communications Technology Planning \& Evaluation (IITP) grant funded by the Korea government(MSIT) (No.2019-0-00075, Artificial Intelligence Graduate School Program(KAIST)).

{\small
\bibliographystyle{ieee_fullname}
\bibliography{egbib}
}

\clearpage

\onecolumn

\renewcommand\thesection{\Alph{section}}
\setcounter{section}{0}
\section{Appendix}

\subsection{Related Work}
\textbf{Continual learning.}
Many penalty-based approaches have been proposed to overcome catastrophic forgetting.
\cite{kirkpatrick2017overcoming} protects the source task performance by a quadratic penalty loss where the importance of each weight is measured by the diagonal of Fisher.
\cite{liu2018rotate} proposes a network reparameterization technique that approximately diagonalizes the Fisher Information Matrix of the network parameters.
In \cite{ritter2018online}, the block diagonal K-FAC is used for a quadratic penalty loss to take interaction between parameters in each single layer into account.
\cite{aljundi2018memory} proposes to measure the importance of a parameter by the magnitude of the gradient.
\cite{zenke2017continual} also defines a quadratic penalty loss designed with the change in loss over an entire trajectory of parameters.
\cite{park2019continual} approximates a true loss function using an asymmetric quadratic function with one of its sides overestimated.

\subsection{The Hessian of a linear layer} \label{app:hessian}

As $(h_l)_{i,m} = \sum_{k} (W_l)_{i,k} (\bar{a}_{l-1})_{k,m}$,
\begin{equation}
\frac{\partial \mathcal{L}_n}{\partial (W_l)_{a,b}} = \sum_{m,i} \frac{\partial \mathcal{L}_n}{\partial (h_l)_{i,m}} \frac{\partial (h_l)_{i,m}}{\partial (W_l)_{a,b}} = \sum_{m} \frac{\partial \mathcal{L}_n}{\partial (h_l)_{a,m}} (\bar{a}_{l-1})_{b,m} .
\end{equation}
Then,
\begin{equation}
\frac{\partial^2 \mathcal{L}_n}{\partial (W_{l'})_{c,d} \partial (W_{l})_{a,b}} = \sum_{m} \left( \frac{\partial}{\partial (W_{l'})_{c,d}}  \left( \frac{\partial \mathcal{L}_n}{\partial (h_l)_{a,m}} \right) (\bar{a}_{l-1})_{b,m} + \frac{\partial \mathcal{L}_n}{\partial (h_l)_{a,m}} \frac{\partial (\bar{a}_{l-1})_{b,m}}{\partial (W_{l'})_{c,d}} \right) .
\end{equation}
Using the chain rule,
\begin{equation} \begin{split}
\frac{\partial}{\partial (W_{l'})_{c,d}}  \left( \frac{\partial \mathcal{L}_n}{\partial (h_l)_{a,m}} \right) & = \sum_{m',i} \frac{\partial^2 \mathcal{L}_n}{\partial (h_{l'})_{i,m'} \partial (h_{l})_{a,m}} \frac{\partial (h_{l'})_{i,m'}}{\partial (W_{l'})_{c,d}}\\
& = \sum_{m'} \frac{\partial^2 \mathcal{L}_n}{\partial (h_{l'})_{c,m'} \partial (h_l)_{a,m}} (\bar{a}_{l'-1})_{d,m'} .
\end{split} \end{equation}
Here, as in \cite{bishop1992exact, popa2014exact}, we can assume $l \le l'$ by the symmetry of Hessian, so
\begin{equation}
\frac{\partial (\bar{a}_{l-1})_{b,m}}{\partial (W_{l'})_{c,d}} = 0 ,
\end{equation}
since $\bar{a}_{l-1}$ is a function of $W_1$, $W_2$, $\cdots$, $W_{l-1}$, but does not depend on $W_l$, $W_{l+1}$, $\cdots$, $W_L$.
Therefore,
\begin{equation}
\frac{\partial^2 \mathcal{L}_n}{\partial (W_l)_{a,b} \partial (W_{l'})_{c,d}} = \sum_{m,m'} (\bar{a}_{l-1})_{b,m} (\bar{a}_{l'-1})_{d,m'} \frac{\partial^2 \mathcal{L}_n}{\partial (h_l)_{a,m} \partial (h_{l'})_{c,m'}} .
\end{equation}

\subsection{Extended K-FAC} \label{app:bnsimple}

We divide the summands of
\begin{equation}
\mathbb{E}_{(x, t)} \left[ \mathbb{E}_n \left[ \sum_{m,m'} (\bar{a}_{l-1})_{b,m} (\bar{a}_{l'-1})_{d,m'} \frac{\partial^2 \mathcal{L}_n}{\partial (h_l)_{a,m} \partial (h_{l'})_{c,m'}} \right]\right]
\end{equation}
into the five groups so that it can be expressed as $G_1+G_2+G_3+G_4+G_5$.
For the derivation, we define
\begin{equation}
(\{\mathcal{H}'''\}_{l,l'})_{a,c} = \mathbb{E}_{(x, t)} \left[\mathbb{E}_n \left[  \sum_{m} \frac{\partial^2 \mathcal{L}_n}{\partial (h_l)_{a,n} \partial (h_{l'})_{c,m}} \right]\right] .
\end{equation}
Note that $\mathcal{H}'''$ is not symmetric in general unlike the others.
In addition, let $\pi$ and $\pi'$ be permutations of $\{1, 2, \cdots, N\}$ such that $n\neq \pi(n) \neq \pi'(n) \neq n$ for all $n$.

(i) $G_1$: $m=m'=n$
\begin{multline}
\mathbb{E}_{(x, t)} \left[ \mathbb{E}_n \left[  (\bar{a}_{l-1})_{b,n} (\bar{a}_{l'-1})_{d,n} \frac{\partial^2 \mathcal{L}_n}{\partial (h_l)_{a,n} \partial (h_{l'})_{c,n}} \right]\right]\\
= \mathbb{E}_{x} \left[ \mathbb{E}_n \left[  (\bar{a}_{l-1})_{b,n} (\bar{a}_{l'-1})_{d,n} \right]\right] \mathbb{E}_{(x, t)} \left[ \mathbb{E}_n \left[ \frac{\partial^2 \mathcal{L}_n}{\partial (h_l)_{a,n} \partial (h_{l'})_{c,n}} \right]\right]
= (\{\bar{A}\}_{l,l'})_{b,d} (\{\mathcal{H}\}_{l,l'})_{a,c} 
\end{multline}
\begin{equation}
\therefore G_1 = \bar{A} * \mathcal{H}
\end{equation}

(ii) $G_2$: $m=m' \neq n$
\begin{equation} \begin{split}
\mathbb{E}_{(x, t)} &\left[ \mathbb{E}_n \left[ \sum_{m(\neq n)} (\bar{a}_{l-1})_{b,m} (\bar{a}_{l'-1})_{d,m} \frac{\partial^2 \mathcal{L}_n}{\partial (h_l)_{a,m} \partial (h_{l'})_{c,m}} \right]\right]\\
&= (N-1) \mathbb{E}_{(x, t)} \left[ \mathbb{E}_n \left[ (\bar{a}_{l-1})_{b,\pi(n)} (\bar{a}_{l'-1})_{d,\pi(n)} \frac{\partial^2 \mathcal{L}_n}{\partial (h_l)_{a,\pi(n)} \partial (h_{l'})_{c,\pi(n)}} \right]\right] \\
& = \mathbb{E}_{x} \left[ \mathbb{E}_n \left[ (\bar{a}_{l-1})_{b,\pi(n)} (\bar{a}_{l'-1})_{d,\pi(n)} \right]\right]  \mathbb{E}_{(x, t)} \left[ \mathbb{E}_n \left[ (N-1)  \frac{\partial^2 \mathcal{L}_n}{\partial (h_l)_{a,\pi(n)} \partial (h_{l'})_{c,\pi(n)}} \right]\right]\\
& = (\{\bar{A}\}_{l,l'})_{b,d} \mathbb{E}_{(x, t)} \left[\mathbb{E}_n \left[  \sum_{m(\neq n)} \frac{\partial^2 \mathcal{L}_n}{\partial (h_l)_{a,m} \partial (h_{l'})_{c,m}} \right]\right]
\end{split} \end{equation}
since for any $m \neq n$,
\begin{equation}
\mathbb{E}_{(x, t)} \left[ \frac{\partial^2 \mathcal{L}_n}{\partial (h_l)_{a,\pi(n)} \partial (h_{l'})_{c,\pi(n)}} \right] = \mathbb{E}_{(x, t)} \left[  \frac{\partial^2 \mathcal{L}_n}{\partial (h_l)_{a,m} \partial (h_{l'})_{c,m}} \right] .
\end{equation}
Though $\mathbb{E}_{(x, t)} \left[ \mathbb{E}_n \left[ (N-1)  \frac{\partial^2 \mathcal{L}_n}{\partial (h_l)_{a,\pi(n)} \partial (h_{l'})_{c,\pi(n)}} \right]\right]$ and $\mathbb{E}_{(x, t)} \left[\mathbb{E}_n \left[  \sum_{m(\neq n)} \frac{\partial^2 \mathcal{L}_n}{\partial (h_l)_{a,m} \partial (h_{l'})_{c,m}} \right]\right]$ are equivalent, the latter is more efficient in estimating the expectation as all possible combinations in a single backward pass are considered.
We use this type of efficient reformulation for other groups as well.
\begin{equation} \begin{split}
\mathbb{E}_{(x, t)} & \left[ \mathbb{E}_n \left[ (N-1)  \frac{\partial^2 \mathcal{L}_n}{\partial (h_l)_{a,\pi(n)} \partial (h_{l'})_{c,\pi(n)}} \right]\right] = \mathbb{E}_{(x, t)} \left[\mathbb{E}_n \left[  \sum_{m(\neq n)} \frac{\partial^2 \mathcal{L}_n}{\partial (h_l)_{a,m} \partial (h_{l'})_{c,m}} \right]\right]\\
&= \mathbb{E}_{(x,t)} \left[ \mathbb{E}_n \left[   \sum_{m} \frac{\partial^2 \mathcal{L}_n}{\partial (h_l)_{a,m} \partial (h_{l'})_{c,m}} - \frac{\partial^2 \mathcal{L}_n}{\partial (h_l)_{a,n} \partial (h_{l'})_{c,n}} \right]\right]
= (\{\mathcal{H}' - \mathcal{H}\}_{l,l'})_{a,c}
\end{split} \end{equation}
\begin{equation}
\therefore G_2 = \bar{A} * (\mathcal{H}' - \mathcal{H})
\end{equation}

(iii) $G_3$: $m=n \neq m'$
\begin{equation} \begin{split}
\mathbb{E}_{(x, t)} & \left[ \mathbb{E}_n \left[ \sum_{m(\neq n)} (\bar{a}_{l-1})_{b,n} (\bar{a}_{l'-1})_{d,m} \frac{\partial^2 \mathcal{L}_n}{\partial (h_l)_{a,n} \partial (h_{l'})_{c,m}} \right]\right]\\
&= (N-1) \mathbb{E}_{(x, t)} \left[ \mathbb{E}_n \left[ (\bar{a}_{l-1})_{b,n} (\bar{a}_{l'-1})_{d,\pi(n)} \frac{\partial^2 \mathcal{L}_n}{\partial (h_l)_{a,n} \partial (h_{l'})_{c,\pi(n)}} \right]\right]\\
&= \mathbb{E}_{x} \left[ \mathbb{E}_n \left[ (\bar{a}_{l-1})_{b,n} (\bar{a}_{l'-1})_{d,\pi(n)}\right]\right]\mathbb{E}_{(x, t)} \left[ \mathbb{E}_n \left[ (N-1) \frac{\partial^2 \mathcal{L}_n}{\partial (h_l)_{a,n} \partial (h_{l'})_{c,\pi(n)}} \right]\right] \\
&= \mathbb{E}_{x} \left[ \mathbb{E}_n \left[  \mathbb{E}_{m(\neq n)} \left[  (\bar{a}_{l-1})_{b,n} (\bar{a}_{l'-1})_{d,m} \right]\right]\right] \mathbb{E}_{(x, t)} \left[\mathbb{E}_n \left[  \sum_{m(\neq n)} \frac{\partial^2 \mathcal{L}_n}{\partial (h_l)_{a,n} \partial (h_{l'})_{c,m}} \right]\right]
\end{split} \end{equation}
\begin{equation} \begin{split}
\mathbb{E}_{x} &\left[ \mathbb{E}_n \left[  \mathbb{E}_{m(\neq n)} \left[  (\bar{a}_{l-1})_{b,n} (\bar{a}_{l'-1})_{d,m} \right]\right]\right]\\
& = \mathbb{E}_{x} \left[ \mathbb{E}_n \left[ \frac{1}{N-1} \left( N \mathbb{E}_{m} \left[  (\bar{a}_{l-1})_{b,n} (\bar{a}_{l'-1})_{d,m} \right] - (\bar{a}_{l-1})_{b,n} (\bar{a}_{l'-1})_{d,n} \right)\right]\right]\\
& = \frac{1}{N-1} \left( N \mathbb{E}_x \left[ \mathbb{E}_n \left[(\bar{a}_{l-1})_{b,n} \right] \mathbb{E}_m \left[(\bar{a}_{l'-1})_{d,m} \right] \right] - (\{\bar{A}\}_{l,l'})_{b,d} \right)\\
& = \frac{1}{N-1} ( \{ N\bar{A}' - \bar{A} \}_{l,l'} )_{b,d}
\end{split} \end{equation}
\begin{equation} \begin{split}
\mathbb{E}_{(x, t)} &\left[\mathbb{E}_n \left[  \sum_{m(\neq n)} \frac{\partial^2 \mathcal{L}_n}{\partial (h_l)_{a,n} \partial (h_{l'})_{c,m}} \right]\right]\\
& = \mathbb{E}_{(x, t)} \left[\mathbb{E}_n \left[  \sum_{m} \frac{\partial^2 \mathcal{L}_n}{\partial (h_l)_{a,n} \partial (h_{l'})_{c,m}} - \frac{\partial^2 \mathcal{L}_n}{\partial (h_l)_{a,n} \partial (h_{l'})_{c,n}} \right]\right] = (\{\mathcal{H}''' - \mathcal{H}\}_{l,l'})_{a,c}
\end{split} \end{equation}
\begin{equation}
\therefore G_3 = \frac{1}{N-1}(N\bar{A}' - \bar{A}) * (\mathcal{H}''' - \mathcal{H})
\end{equation}

(iv) $G_4$: $m \neq n = m'$
\begin{multline}
\mathbb{E}_{(x, t)} \left[ \mathbb{E}_n \left[ \sum_{m(\neq n)} (\bar{a}_{l-1})_{b,m} (\bar{a}_{l'-1})_{d,n} \frac{\partial^2 \mathcal{L}_n}{\partial (h_l)_{a,m} \partial (h_{l'})_{c,n}} \right]\right] \\
= \mathbb{E}_{(x, t)} \left[ \mathbb{E}_n \left[ \sum_{m(\neq n)} (\bar{a}_{l'-1})_{d,n} (\bar{a}_{l-1})_{b,m} \frac{\partial^2 \mathcal{L}_n}{\partial (h_{l'})_{c,n} \partial (h_{l})_{a,m}} \right]\right]
\end{multline}
\begin{equation}
\therefore G_4 = G_3^\top = \frac{1}{N-1}(N\bar{A}' - \bar{A}) * (\mathcal{H}'''^\top - \mathcal{H})
\end{equation}

(v) $G_5$: $n \neq m \neq m' \neq n$
\begin{equation} \begin{split}
\mathbb{E}_{(x, t)} &\left[ \mathbb{E}_n \left[ \sum_{\substack{m,m'\\(n \neq m \neq m' \neq n)}} (\bar{a}_{l-1})_{b,m} (\bar{a}_{l'-1})_{d,m'} \frac{\partial^2 \mathcal{L}_n}{\partial (h_l)_{a,m} \partial (h_{l'})_{c,m'}} \right]\right]\\
&= (N-1)(N-2) \mathbb{E}_{(x, t)} \left[ \mathbb{E}_n \left[ (\bar{a}_{l-1})_{b,\pi(n)} (\bar{a}_{l'-1})_{d,\pi'(n)} \frac{\partial^2 \mathcal{L}_n}{\partial (h_l)_{a,\pi(n)} \partial (h_{l'})_{c,\pi'(n)}} \right]\right]\\
&= \mathbb{E}_{x} \left[ \mathbb{E}_n \left[ (\bar{a}_{l-1})_{b,\pi(n)} (\bar{a}_{l'-1})_{d,\pi'(n)} \right]\right] \mathbb{E}_{(x, t)} \left[ \mathbb{E}_n \left[ (N-1)(N-2) \frac{\partial^2 \mathcal{L}_n}{\partial (h_l)_{a,\pi(n)} \partial (h_{l'})_{c,\pi'(n)}} \right]\right] \\
&= \mathbb{E}_{x} \left[ \mathbb{E}_n \left[  \mathbb{E}_{\substack{m,m'\\(n \neq m \neq m' \neq n)}} \left[  (\bar{a}_{l-1})_{b,m} (\bar{a}_{l'-1})_{d,m'} \right]\right]\right] \mathbb{E}_{(x, t)} \left[\mathbb{E}_n \left[  \sum_{\substack{m,m'\\(n \neq m \neq m' \neq n)}} \frac{\partial^2 \mathcal{L}_n}{\partial (h_l)_{a,m} \partial (h_{l'})_{c,m'}} \right]\right]
\end{split} \end{equation}
\begin{equation} \begin{split}
\mathbb{E}_{x} &\left[ \mathbb{E}_n \left[  \mathbb{E}_{\substack{m,m'\\(n \neq m \neq m' \neq n)}} \left[  (\bar{a}_{l-1})_{b,m} (\bar{a}_{l'-1})_{d,m'} \right]\right]\right]\\
& = \mathbb{E}_x \bigg[ \mathbb{E}_n \bigg[ \frac{1}{(N-1)(N-2)} \Big(  N^2 \mathbb{E}_{m,m'} [(\bar{a}_{l-1})_{b,m} (\bar{a}_{l'-1})_{d,m'}] \\
& \qquad - N \mathbb{E}_m [(\bar{a}_{l-1})_{b,n} (\bar{a}_{l'-1})_{d,m}] - N \mathbb{E}_m [(\bar{a}_{l-1})_{b,m} (\bar{a}_{l'-1})_{d,n}] + (\bar{a}_{l-1})_{b,n} (\bar{a}_{l'-1})_{d,n} \\
& \qquad - N \mathbb{E}_m [(\bar{a}_{l-1})_{b,m} (\bar{a}_{l'-1})_{d,m}] + (\bar{a}_{l-1})_{b,n} (\bar{a}_{l'-1})_{d,n} \Big)\bigg]\bigg]\\
& = \frac{1}{(N-1)(N-2)} (\{(N^2-2N)\bar{A}' - (N-2)\bar{A}\}_{l,l'})_{b,d} = \frac{1}{N-1} ( \{ N\bar{A}' - \bar{A} \}_{l,l'} )_{b,d}
\end{split} \end{equation}
\begin{equation} \begin{split}
\mathbb{E}_{(x, t)} & \left[\mathbb{E}_n \left[  \sum_{\substack{m,m'\\(n \neq m \neq m' \neq n)}} \frac{\partial^2 \mathcal{L}_n}{\partial (h_l)_{a,m} \partial (h_{l'})_{c,m'}} \right]\right]\\
&= \mathbb{E}_{(x,t)} \Bigg[ \mathbb{E}_n \Bigg[   \sum_{m,m'} \frac{\partial^2 \mathcal{L}_n}{\partial (h_l)_{a,m} \partial (h_{l'})_{c,m'}}\\
& \qquad - \sum_{m} \frac{\partial^2 \mathcal{L}_n}{\partial (h_l)_{a,n} \partial (h_{l'})_{c,m}} - \sum_{m} \frac{\partial^2 \mathcal{L}_n}{\partial (h_l)_{a,m} \partial (h_{l'})_{c,n}} + \frac{\partial^2 \mathcal{L}_n}{\partial (h_l)_{a,n} \partial (h_{l'})_{c,n}}\\
& \qquad - \sum_{m} \frac{\partial^2 \mathcal{L}_n}{\partial (h_l)_{a,m} \partial (h_{l'})_{c,m}} + \frac{\partial^2 \mathcal{L}_n}{\partial (h_l)_{a,n} \partial (h_{l'})_{c,n}}  \Bigg] \Bigg]\\
& = (\{\mathcal{H}'' - \mathcal{H}''' - \mathcal{H}'''^\top - \mathcal{H}' + 2\mathcal{H}\}_{l,l'})_{a,c}
\end{split} \end{equation}
\begin{equation}
\therefore G_5 = \frac{1}{N-1}(N\bar{A}' - \bar{A}) * (\mathcal{H}'' - \mathcal{H}''' - \mathcal{H}'''^\top - \mathcal{H}' + 2\mathcal{H})
\end{equation}

Therefore, for $N>2$,
\begin{equation} \begin{split}
H & = G_1 + G_2 + G_3 + G_4 + G_5\\
& = \bar{A} * \mathcal{H} + \bar{A} * (\mathcal{H}' - \mathcal{H}) + \frac{1}{N-1}(N\bar{A}' - \bar{A}) * (\mathcal{H}''' + \mathcal{H}'''^\top - 2\mathcal{H}) \\
& \qquad + \frac{1}{N-1}(N\bar{A}' - \bar{A}) * (\mathcal{H}'' - \mathcal{H}''' - \mathcal{H}'''^\top - \mathcal{H}' + 2\mathcal{H}) \\
& = \bar{A} * \mathcal{H}' + \frac{1}{N-1} \left( N \bar{A}' -  \bar{A} \right) * (\mathcal{H}'' - \mathcal{H}')
\end{split} \end{equation}

For $N=2$,
\begin{equation} \begin{split}
H = G_1 + G_2 + G_3 + G_4 & = \bar{A} * \mathcal{H}' + \frac{1}{N-1}(N\bar{A}' - \bar{A}) * (\mathcal{H}''' + \mathcal{H}'''^\top - 2\mathcal{H})\\
& = \bar{A} * \mathcal{H}' + \frac{1}{N-1}(N\bar{A}' - \bar{A}) * (\mathcal{H}'' - \mathcal{H}')
\end{split} \end{equation}
since
\begin{multline}
(\{\mathcal{H}'' - \mathcal{H}'\}_{l,l'})_{a,c} = (\{\mathcal{H}''' + \mathcal{H}'''^\top - 2\mathcal{H}\}_{l,l'})_{a,c} \\
= \frac{1}{2} \mathbb{E}_{(x,t)} \left[ \frac{\partial^2 \mathcal{L}_1}{\partial (h_l)_{a,1} \partial (h_{l'})_{c,2}} + \frac{\partial^2 \mathcal{L}_1}{\partial (h_l)_{a,2} \partial (h_{l'})_{c,1}} + \frac{\partial^2 \mathcal{L}_2}{\partial (h_l)_{a,1} \partial (h_{l'})_{c,2}} + \frac{\partial^2 \mathcal{L}_2}{\partial (h_l)_{a,2} \partial (h_{l'})_{c,1}} \right] .
\end{multline}

For $N=1$,
\begin{equation}
H = G_1 = \bar{A} * \mathcal{H} = \bar{A} * \mathcal{H}' + \left( N \bar{A}' -  \bar{A} \right) * (\mathcal{H}'' - \mathcal{H}')
\end{equation}
since $\mathcal{H} = \mathcal{H}' = \mathcal{H}''$.

\subsection{Positive semi-definiteness of XK-FAC} \label{app:psd}

If we denote the $n$-th column vector of $\bar{a}_{l-1}$ by $(\bar{a}_{l-1})_{:,n}$, then
\begin{equation}
\{\bar{A}\}_{l,l'} = \mathbb{E}_x [\mathbb{E}_n [(\bar{a}_{l-1})_{:,n} (\bar{a}_{l'-1})_{:,n}^\top]] ,
\end{equation}
so
\begin{equation}
\bar{A} = \mathbb{E}_x [\mathbb{E}_n [(\bar{a}_{0:L-1})_{:,n} (\bar{a}_{0:L-1})_{:,n}^\top]] \succeq 0
\label{eq:appA}
\end{equation}
where $(\bar{a}_{0:L-1})_{:,n} = \begin{bmatrix} (\bar{a}_0)_{:,n}^\top & (\bar{a}_1)_{:,n}^\top & \cdots & (\bar{a}_{L-1})_{:,n}^\top \end{bmatrix}^\top$.
For $\bar{A}'$,
\begin{equation}
\bar{A}' = \mathbb{E}_x [\mathbb{E}_n [(\bar{a}_{0:L-1})_{:,n}] \mathbb{E}_n [(\bar{a}_{0:L-1})_{:,n}]^\top ] \succeq 0 .
\label{eq:appA'}
\end{equation}
Also,
\begin{equation}
\bar{A} - \bar{A}' = \mathbb{E}_x [\mathbb{E}_n [(\bar{a}_{0:L-1})_{:,n} (\bar{a}_{0:L-1})_{:,n}^\top] - \mathbb{E}_n [(\bar{a}_{0:L-1})_{:,n}] \mathbb{E}_n [(\bar{a}_{0:L-1})_{:,n}]^\top ] \succeq 0
\end{equation}
since it is an expectation of covariance matrices.
Thus,
\begin{equation}
\bar{A} \succeq 0, \quad \bar{A}' \succeq 0, \quad \bar{A} - \bar{A}' \succeq 0 .
\end{equation}

Similarly, $\hat{\mathcal{H}}' \succeq 0$, $\hat{\mathcal{H}}'' \succeq 0$, and $N\hat{\mathcal{H}}' - \hat{\mathcal{H}}'' \succeq 0$, because
\begin{align}
\hat{\mathcal{H}}' & = \mathbb{E}_{(x,y)} \left[ \mathbb{E}_n \left[ N \mathbb{E}_m \left[ \frac{\partial \mathcal{L}_n}{\partial (h_{1:L})_{:,m}}  \frac{\partial \mathcal{L}_n}{\partial (h_{1:L})_{:,m}}^\top \right] \right] \right] \succeq 0 , 
\label{eq:appH'}\\
\hat{\mathcal{H}}'' & = \mathbb{E}_{(x,y)} \left[ \mathbb{E}_n \left[ N^2 \mathbb{E}_m \left[ \frac{\partial \mathcal{L}_n}{\partial (h_{1:L})_{:,m}} \right] \mathbb{E}_m \left[ \frac{\partial \mathcal{L}_n}{\partial (h_{1:L})_{:,m}} \right]^\top \right] \right] \succeq 0 , 
\label{eq:appH''}\\
N\hat{\mathcal{H}}' - \hat{\mathcal{H}}'' & = N^2 \mathbb{E}_{(x,y)} \Bigg[ \mathbb{E}_n \Bigg[ \mathbb{E}_m \left[ \frac{\partial \mathcal{L}_n}{\partial (h_{1:L})_{:,m}}  \frac{\partial \mathcal{L}_n}{\partial (h_{1:L})_{:,m}}^\top \right] - \mathbb{E}_m \left[ \frac{\partial \mathcal{L}_n}{\partial (h_{1:L})_{:,m}} \right] \mathbb{E}_m \left[ \frac{\partial \mathcal{L}_n}{\partial (h_{1:L})_{:,m}} \right]^\top \Bigg] \Bigg] \succeq 0 .
\end{align}

Therefore, for $N\ge 2$,
\begin{equation}
H = \bar{A} * \hat{\mathcal{H}}' + \frac{1}{N-1} \left( N \bar{A}' -  \bar{A} \right) * (\hat{\mathcal{H}}'' - \hat{\mathcal{H}}') = \frac{1}{N-1}(\bar{A}-\bar{A}')*(N\hat{\mathcal{H}}'-\hat{\mathcal{H}}'') + \bar{A}'*\hat{\mathcal{H}}'' \succeq 0
\end{equation}
since the Khatri--Rao product of two symmetrically partitioned positive semi-definite matrices is positive semi-definite~\cite{liu1999matrix, zhang2002inequalities}.
For $N=1$, $H = \bar{A}*\hat{\mathcal{H}}' \succeq 0$.

Also, this proof reveals how to implement XK-FAC efficiently.
Computing each block matrices in XK-FAC, $\bar{A}$, $\bar{A}'$, $\hat{\mathcal{H}}'$, and $\hat{\mathcal{H}}''$, just requires some averages and matrix-matrix multiplications (Equations~\ref{eq:appA}, \ref{eq:appA'}, \ref{eq:appH'}, \ref{eq:appH''}).
Thus, XK-FAC can be implemented very easily using any basic linear algebra subprograms (BLAS).

\subsection{Combining XK-FAC and KFC} \label{app:kfc}

For the $l$-th convolutional layer, let $a_{l-1}$ of size $C_{l-1} \times (S_{l-1} N)$ be an input to the layer and $W_{l}$ of size $C_{l} \times (C_{l-1} K_l + 1)$ be the weight, where $S$ and $K$ represent the flattened spatial dimension and kernel dimension, respectively.
A convolution operation can be converted to a matrix-matrix multiplication by unrolling the input~\cite{chellapilla2006high} (this unrolling function is often called \texttt{im2col}).
Let $\bm{a}_{l-1}$ of size $(C_{l-1} K_l) \times (S_l N)$ be the unrolled input of $a_{l-1}$ (denoted by $\llbracket\cdot\rrbracket$ in \cite{grosse2016kronecker}) and $\bar{\bm{a}}_{l-1}$ of size $(C_{l-1} K_l + 1) \times (S_l N)$ be the unrolled input with homogeneous dimension appended (denoted by $\llbracket\cdot\rrbracket_H$ in \cite{grosse2016kronecker}).
Then, the output $h_l$ of size $C_l \times (S_l N)$ is
\begin{equation}
h_l = W_l \bar{\bm{a}}_{l-1} .
\end{equation}
Thus, if the Hessian is approximated by the Fisher information matrix, then Equation~\ref{eq:hessian} becomes
\begin{equation}
\mathbb{E}_{(x, y)} \left[ \mathbb{E}_n \left[ \sum_{s,s',m,m'} (\bar{\bm{a}}_{l-1})_{b,(s,m)} (\bar{\bm{a}}_{l-1})_{d,(s',m')} \frac{\partial \mathcal{L}_n}{\partial (h_l)_{a,(s,m)}} \frac{\partial \mathcal{L}_n}{\partial (h_{l})_{c,(s',m')}} \right]\right] ,
\end{equation}
where $s$ and $s'$ index the spatial location.

KFC assumes three conditions: IAD, SH, and SUD, and these conditions can be straightforwardly extended for a different mini-batches case.
If we apply KFC for each $(m, m')$, we get
\begin{equation}
\sum_{m,m'} \left(\sum_{s} \mathbb{E}_{x} \left[ \mathbb{E}_n \left[  (\bar{\bm{a}}_{l-1})_{b,(s,m)} (\bar{\bm{a}}_{l-1})_{d,(s,m')} \right] \right] \right) \left( \frac{1}{S_l} \sum_{s} \mathbb{E}_{(x, y)} \left[ \mathbb{E}_n \left[ \frac{\partial \mathcal{L}_n}{\partial (h_l)_{a,(s,m)}} \frac{\partial \mathcal{L}_n}{\partial (h_{l})_{c,(s,m')}} \right]\right] \right) .
\end{equation}
Now, the $N^2$ summands here can also be divided into the five groups, and the remaining processes are exactly the same as Appendix~\ref{app:bnsimple}.
Therefore,
\begin{equation}
\{H\}_{l,l} = \{\bar{A}\}_{l,l} \otimes \{\hat{\mathcal{H}}'\}_{l,l} + \frac{1}{\textnormal{max}(N-1, 1)} \left( N \{\bar{A}'\}_{l,l} -  \{\bar{A}\}_{l,l} \right) \otimes (\{\hat{\mathcal{H}}''\}_{l,l} - \{\hat{\mathcal{H}}'\}_{l,l}) ,
\end{equation}
where
\begin{align}
(\{\bar{A}\}_{l,l})_{b,d} & = \sum_{s} \mathbb{E}_{x} \left[ \mathbb{E}_n \left[  (\bar{\bm{a}}_{l-1})_{b,(s,n)} (\bar{\bm{a}}_{l-1})_{d,(s,n)} \right] \right] , \\
(\{\bar{A}'\}_{l,l})_{b,d} & = \sum_{s} \mathbb{E}_{x} \left[ \mathbb{E}_n \left[  (\bar{\bm{a}}_{l-1})_{b,(s,n)} \right] \mathbb{E}_n \left[ (\bar{\bm{a}}_{l-1})_{d,(s,n)} \right] \right] , \\
(\{\hat{\mathcal{H}}'\}_{l,l})_{a,c} & = \frac{1}{S_l} \sum_{s} \mathbb{E}_{(x, y)} \left[ \mathbb{E}_n \left[ \sum_{m} \frac{\partial \mathcal{L}_n}{\partial (h_l)_{a,(s,m)}} \frac{\partial \mathcal{L}_n}{\partial (h_{l})_{c,(s,m)}} \right]\right] , \\
(\{\hat{\mathcal{H}}''\}_{l,l})_{a,c} & = \frac{1}{S_l} \sum_{s} \mathbb{E}_{(x, y)} \left[ \mathbb{E}_n \left[ \left( \sum_{m} \frac{\partial \mathcal{L}_n}{\partial (h_l)_{a,(s,m)}} \right) \left( \sum_{m} \frac{\partial \mathcal{L}_n}{\partial (h_{l})_{c,(s,m)}} \right) \right]\right] ,
\end{align}
and $\otimes$ is the Kronecker product.

\subsection{Effect of $\alpha_s$ and $\alpha_t$} \label{app:res}

\begin{figure*}[ht]
\centering
\includegraphics[width=0.96\linewidth,page=1]{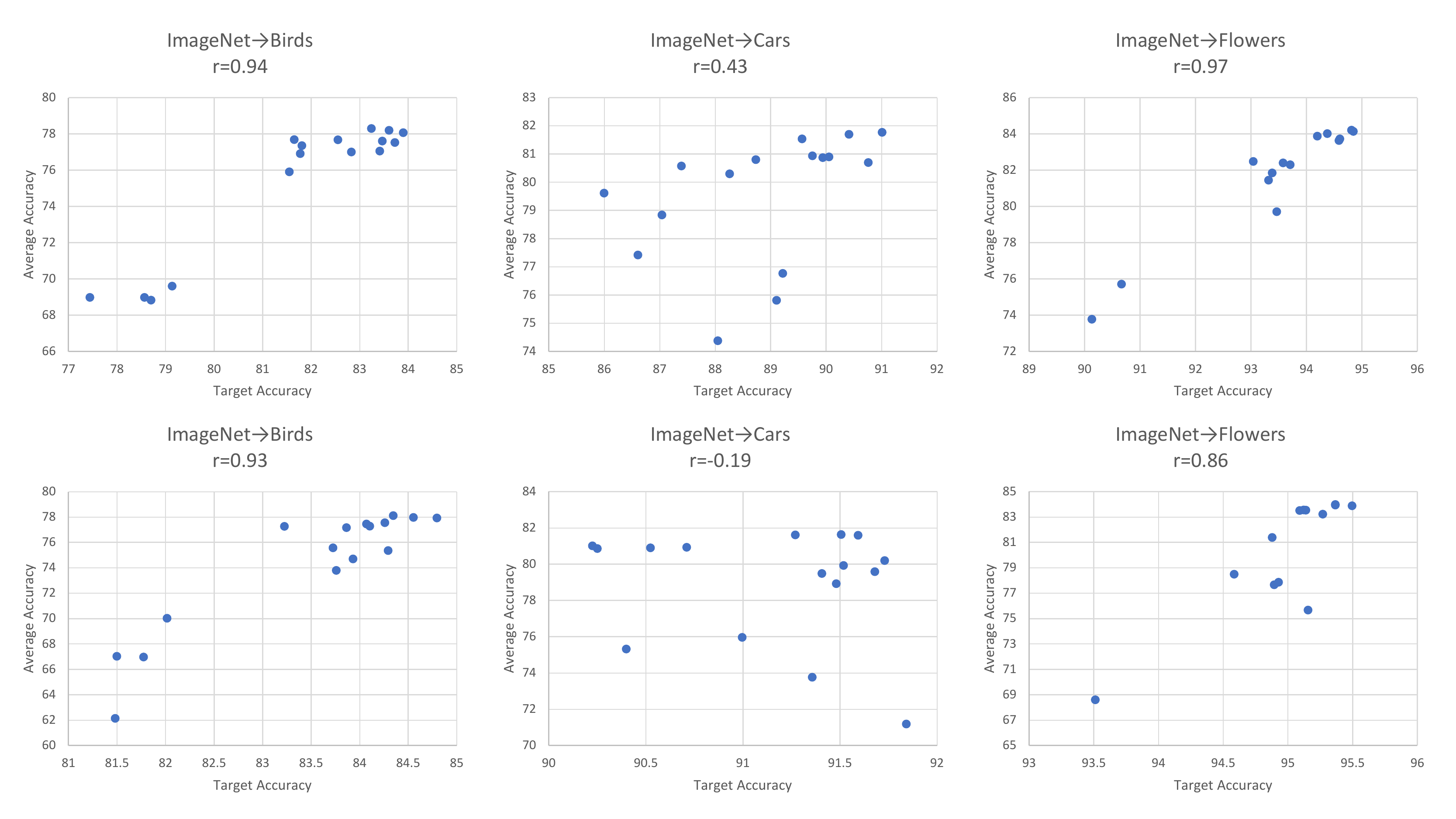}
\caption{
Each point represents the result of a specific hyperparameters (learning rate, damping) setting.
$\alpha_s$ and $\alpha_t$ are used in the top graphs, and they are not used in the bottom graphs.
With $\alpha_s$ and $\alpha_t$, the target accuracy and average accuracy have a more positive correlation.
}
\end{figure*}
\begin{figure*}[ht]
\centering
\includegraphics[width=0.96\linewidth,page=2]{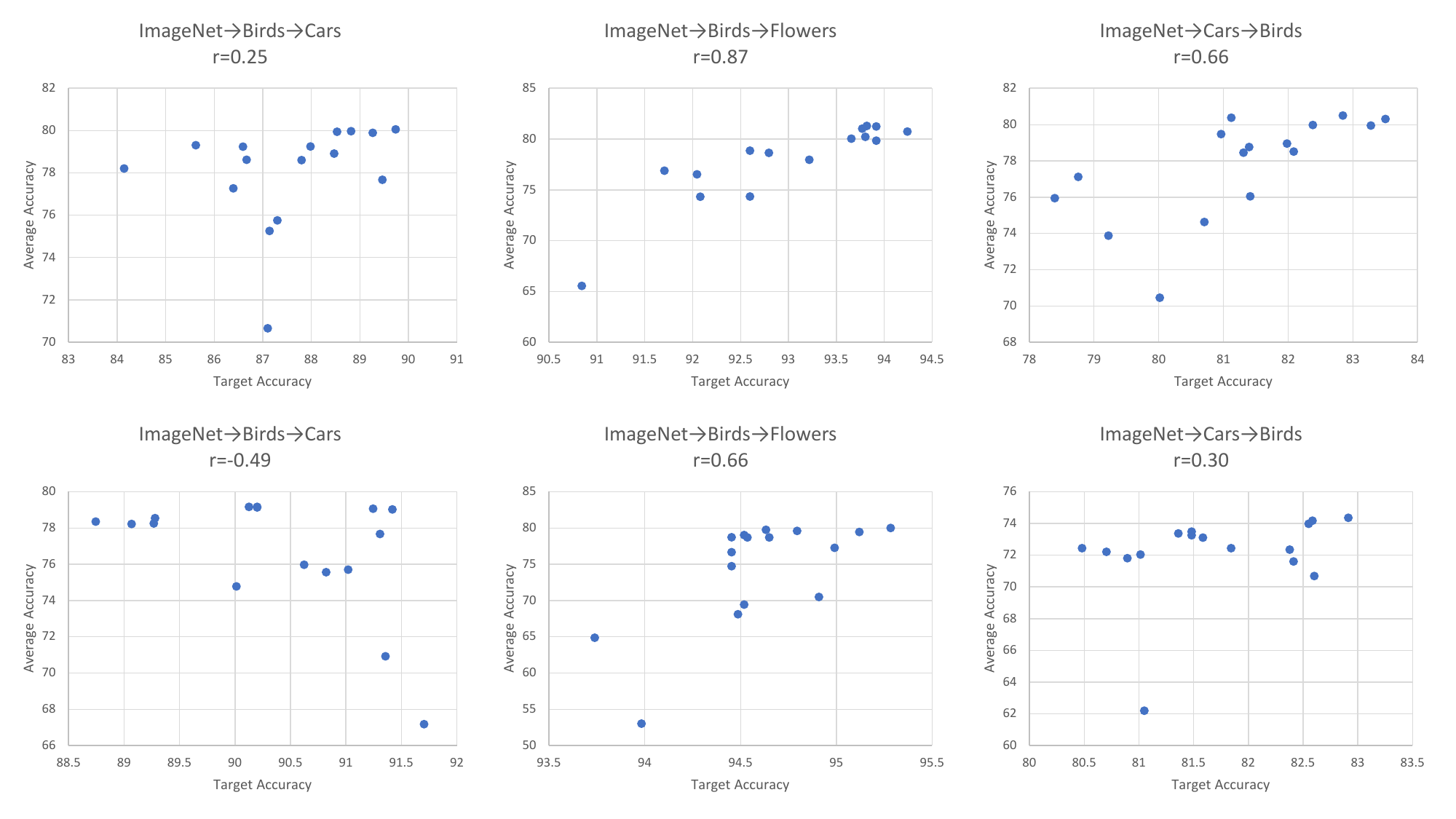}
\caption{
Each point represents the result of a specific hyperparameters (learning rate, damping) setting.
$\alpha_s$ and $\alpha_t$ are used in the top graphs, and they are not used in the bottom graphs.
With $\alpha_s$ and $\alpha_t$, the target accuracy and average accuracy have a more positive correlation.
}
\end{figure*}
\begin{figure*}[ht]
\centering
\includegraphics[width=0.96\linewidth,page=3]{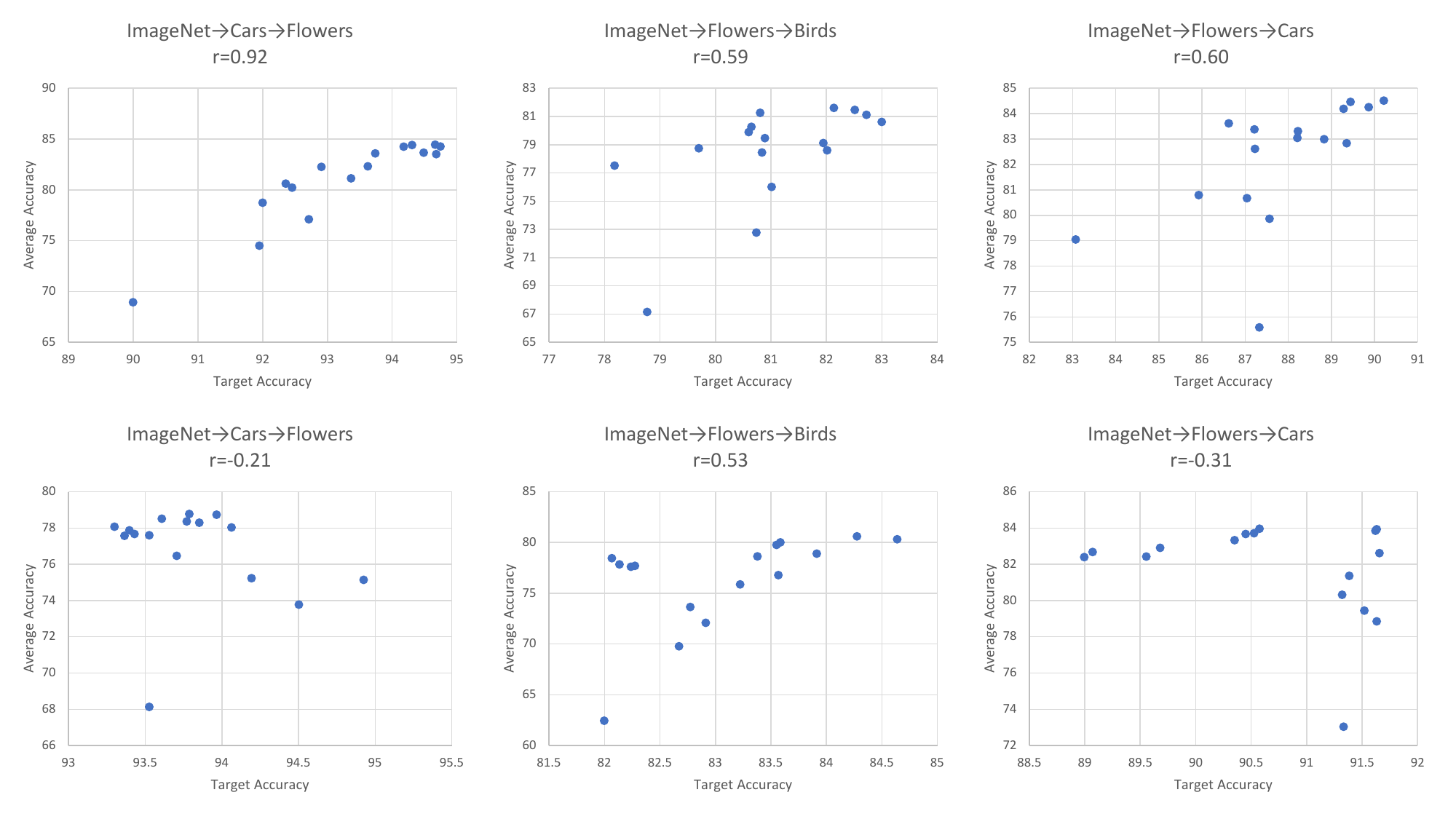}
\caption{
Each point represents the result of a specific hyperparameters (learning rate, damping) setting.
$\alpha_s$ and $\alpha_t$ are used in the top graphs, and they are not used in the bottom graphs.
With $\alpha_s$ and $\alpha_t$, the target accuracy and average accuracy have a more positive correlation.
}
\end{figure*}
\begin{figure*}[ht]
\centering
\includegraphics[width=0.96\linewidth,page=4]{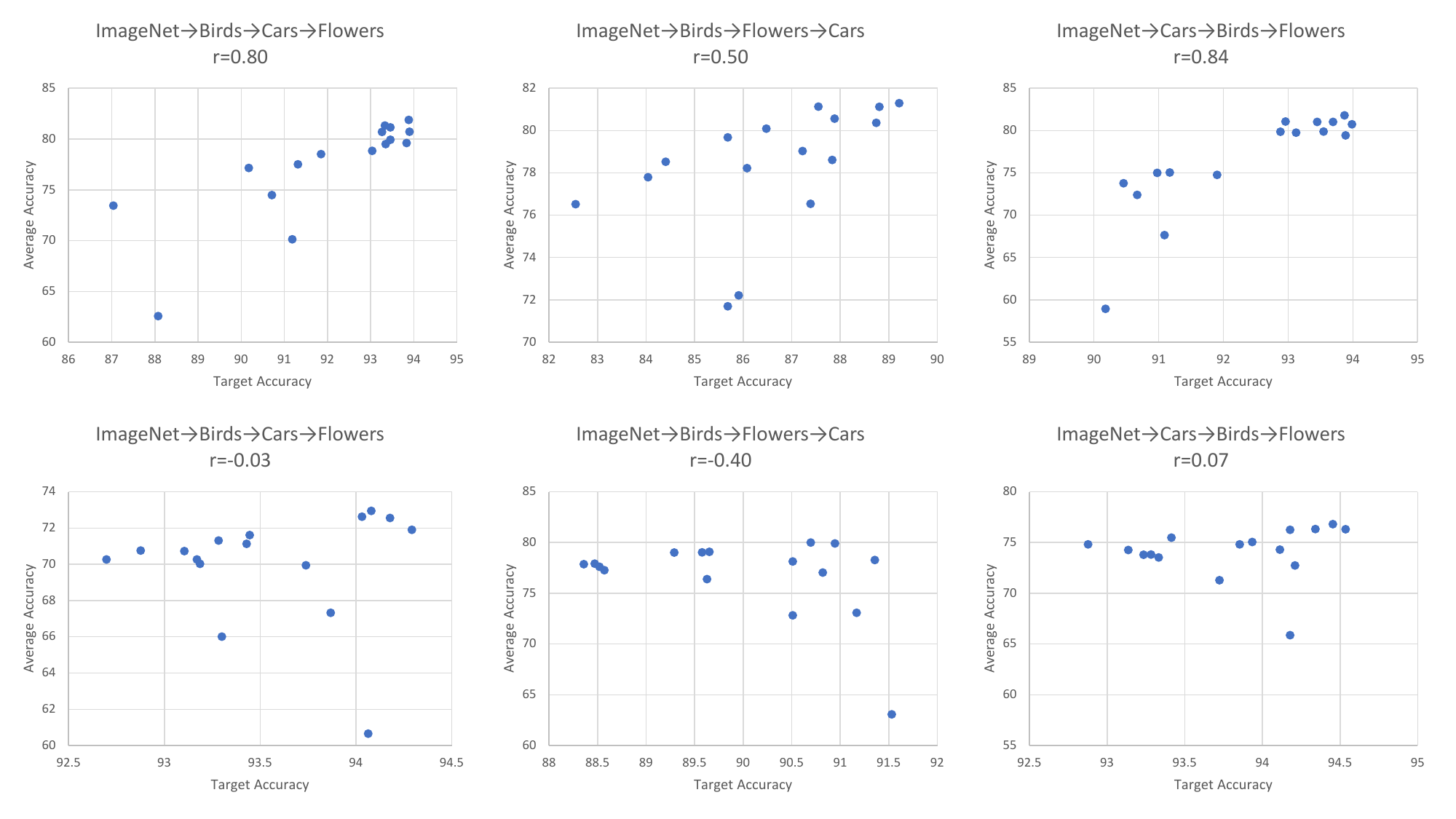}
\caption{
Each point represents the result of a specific hyperparameters (learning rate, damping) setting.
$\alpha_s$ and $\alpha_t$ are used in the top graphs, and they are not used in the bottom graphs.
With $\alpha_s$ and $\alpha_t$, the target accuracy and average accuracy have a more positive correlation.
}
\end{figure*}
\begin{figure*}[ht]
\centering
\includegraphics[width=0.96\linewidth,page=5]{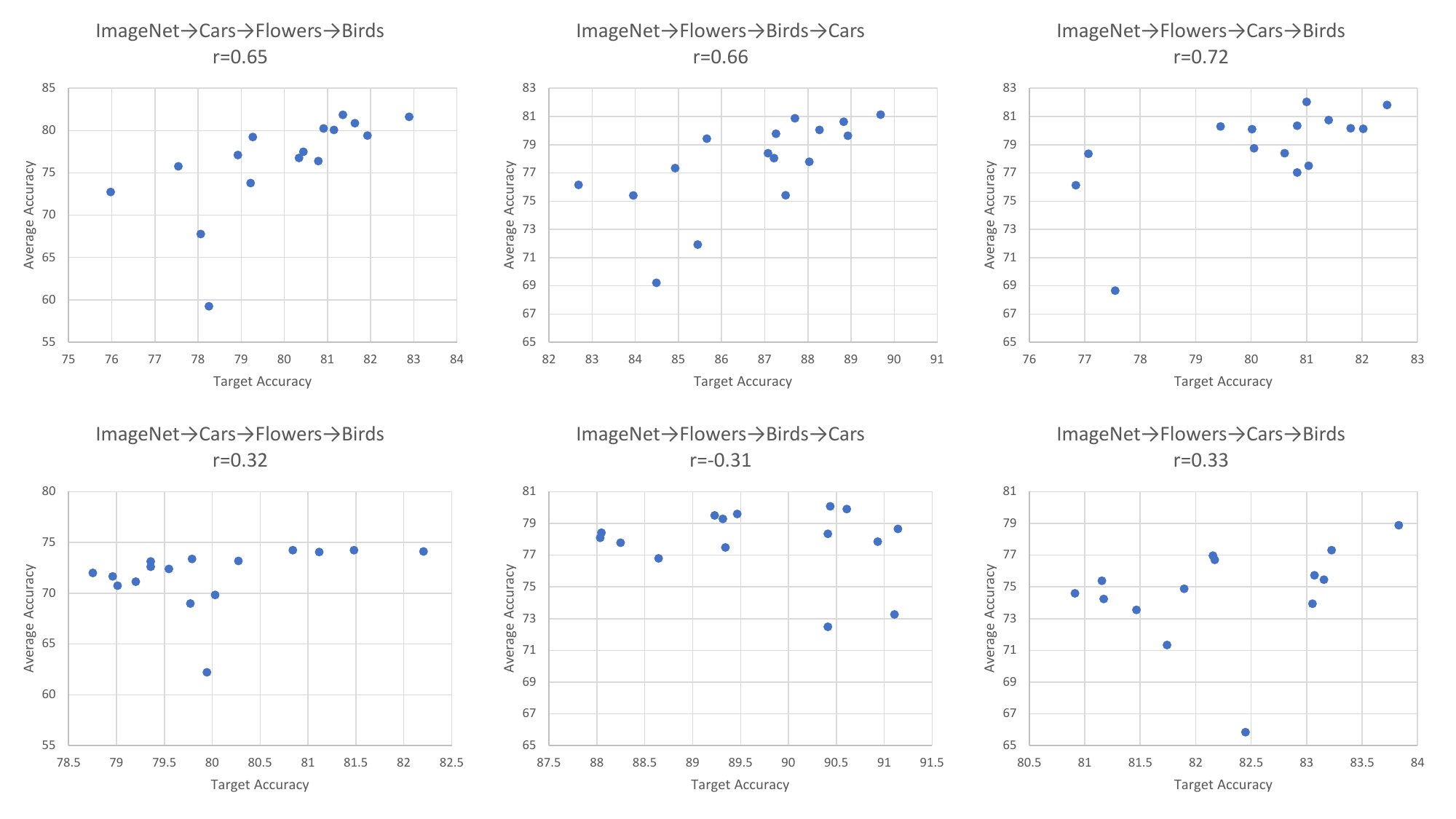}
\caption{
Each point represents the result of a specific hyperparameters (learning rate, damping) setting.
$\alpha_s$ and $\alpha_t$ are used in the top graphs, and they are not used in the bottom graphs.
With $\alpha_s$ and $\alpha_t$, the target accuracy and average accuracy have a more positive correlation.
}
\end{figure*}

\end{document}